\documentclass[11pt]{article}

\usepackage[T1]{fontenc}
\usepackage[utf8]{inputenc}
\usepackage{lmodern}
\usepackage{microtype}

\usepackage[margin=1in]{geometry}

\usepackage{amsmath,amssymb,amsfonts,amsthm,mathtools}

\usepackage{booktabs}
\usepackage{tabularx}

\usepackage{graphicx}
\usepackage{xcolor}

\usepackage{nicefrac}

\usepackage[numbers,sort&compress]{natbib}

\usepackage[hidelinks]{hyperref}

\theoremstyle{plain}
\newtheorem{theorem}{Theorem}
\newtheorem{proposition}{Proposition}
\newtheorem{lemma}{Lemma}
\newtheorem{corollary}{Corollary}

\theoremstyle{definition}
\newtheorem{definition}{Definition}
\newtheorem{assumption}{Assumption}

\theoremstyle{remark}
\newtheorem{remark}{Remark}

\title{In-Context Learning Operates as Concept Subspace Learning}

\author{
 Wei Tang$^{1}$\thanks{Equal contribution.}, 
 \space Xinyan Jiang$^{2*}$, 
 \space Fakhri Karray$^{1}$, 
 \space Lijie Hu$^{1}$ 
 \\
  $^1$Mohamed bin Zayed University of Artificial Intelligence (MBZUAI)\\
  $^2$Shanghai Advanced Research Institute, Chinese Academy of Science\\
  \texttt{\{wei.tang, fakhri.karray, lijie.hu\}@mbzuai.ac.ae}, \texttt{jiangxy2024@sari.ac.cn}
}

\date{}

\begin{document}

\maketitle

\begin{abstract}
Regression and Bayesian accounts of in-context learning (ICL) explain how demonstrations can induce predictors, while mechanistic analyses often identify compact activation directions that steer prompted behavior. However, it remains unclear whether structured demonstrations induce low-dimensional concept inference. We study this question through a concept-subspace view of ICL, in which tasks vary only along intrinsic concept coordinates, although inputs are observed in a high-dimensional ambient space. For ridge and least-squares ICL proxies, prediction decomposes exactly into concept-coordinate regression and off-subspace leakage. Under block-diagonal or near-block-diagonal covariance assumptions, the leading estimation and nuisance-sensitivity terms scale with the dimension of the concept subspace, while residual effects are controlled by cross-subspace coupling. This separation gives a mechanistic prediction: recoverable task information should concentrate in a low-dimensional, task-aligned activation subspace. On CounterFact-derived multi-relation prompts with Llama-3-8B, a $68$--$73$-dimensional subspace of the $4096$-dimensional residual stream restores $78.8\%$ of the clean--corrupted accuracy gap, whereas patching the complementary subspace restores $0\%$. Concept swaps redirect predictions toward injected relations, while random and cross-task matched-rank controls are largely ineffective. Additional experiments on Qwen2.5-7B and a controlled cross-lingual rule task show the same qualitative pattern. These results support concept subspaces as compact, task-aligned mediators of recoverable ICL behavior in structured task families, without implying full-circuit recovery.
\end{abstract}

\section{Introduction}
\label{sec:introduction}
In-context learning (ICL) enables a pretrained language model to infer a new task from a few demonstrations without updating its weights \cite{brown2020language, wu2024glance, wei2023larger}. Demonstrations do more than provide supervised examples: they can specify label spaces, input distributions, output formats, and semantic priors \cite{min2022rethinking,chan2022data}. 
Theoretical work has shown that, in controlled settings, ICL can implement least-squares or ridge regression, Bayesian inference, gradient descent, and algorithm selection~\cite{xie2022explanation,bai2023transformers}. These accounts explain how a predictor can be fit from a prompt, but they often treat the inferred task as an unconstrained vector in a high-dimensional ambient space. 

This issue is most salient when prompts share latent structure. Many demonstrations may instantiate a shared relation, rule, transformation, or semantic attribute~\cite{pan2023what}. In such cases, treating the task as an unconstrained ambient parameter may obscure the relevant geometry. This raises a fundamental structural question: \emph{When tasks share a low-dimensional conceptual structure, does ICL operate by explicitly inferring that low-dimensional concept?} 
This question is related to recent evidence that transformer representations can contain disentangled latent concepts or language-agnostic concept directions \cite{hong2026latentconcept,dumas2025separating}.
The answer has direct statistical and mechanistic consequences.
If ICL exploits shared task geometry, estimation should scale with the concept dimension, while orthogonal perturbations act mainly as nuisances. Diffuse high-dimensional pattern matching gives no reason to expect such dimension-dependent behavior or nuisance stability. Existing regression theories, activation-space analyses, and causal-intervention tools each support parts of this view~\cite{cunningham2024sparse}, but do not identify a common latent object induced by structured ICL prompts.

To make this question precise, we introduce \emph{concept-subspace ICL}. In our model, task parameters satisfy $w = U\beta$, where $U \in \mathbb{R}^{d \times r}$ has orthonormal columns and $r \ll d$. Inputs live in the ambient space, but the task depends on them only through concept coordinates $z = U^\top x$. Directions orthogonal to $U$ are irrelevant in the population model, except through finite-sample effects and covariance leakage. For ridge and least-squares ICL proxies, ambient prediction decomposes into regression over the $r$-dimensional concept coordinates plus an off-subspace leakage term. Under block-diagonal or near-block-diagonal covariance assumptions, the leading estimation and nuisance-sensitivity terms depend on $r$, while residual error is controlled by cross-subspace coupling. The same formulation also admits a Bayesian interpretation as posterior inference over $\beta$ and identifies the concept subspace up to rotation~\cite{stuehmer2020isa,yang2025fastmipl}.

We turn the statistical hypothesis into a mechanistic test. If a transformer implements concept-subspace inference, then task-relevant information should be concentrated in a low-dimensional, task-aligned activation subspace. Patching this subspace from a clean prompt into a corrupted prompt should restore much of the ICL behavior, whereas patching the orthogonal complement or matched-rank control subspaces should have little effect. This gives a falsifiable intervention test: the relevant subspace must carry causal task information, not merely provide a low-rank direction along which the model can be steered.
We empirically evaluate this test on CounterFact-derived multi-relation ICL tasks \cite{meng2022locating} using Llama-3-8B \cite{grattafiori2024llama}. A $68$--$73$-dimensional subspace of the $4096$-dimensional residual stream recovers $78.8\%$ of the clean--corrupted accuracy gap, while patching the complementary subspace restores $0\%$. Concept-swap interventions redirect predictions toward the injected relation, whereas random and cross-task matched-rank controls are largely ineffective. Additional experiments show consistent patterns on cross-lingual tasks~\cite{lample2018word} and on Qwen2.5-7B \cite{qwen_qwen25_2024}. These results do not imply that the full ICL circuit is low-dimensional. Rather, they indicate that, for the structured task families, a compact concept subspace carries a substantial portion of the recoverable causal signal.

Our contributions can be summarized as follows. First, we formulate ICL with shared task variation as estimation in an intrinsic concept subspace rather than unconstrained ambient regression. Second, we analyze ridge and least-squares proxies and show that, under explicit covariance and regularization assumptions, dominant error and nuisance-sensitivity terms scale with the concept dimension $r$, while leakage is controlled by cross-subspace coupling. Third, we turn the theory into a subspace-patching hypothesis and test it with concept-only patching, complementary-space controls, matched-rank controls, and concept swaps on controlled multi-relation tasks.

\section{Related Work}
\label{sec:related}
\paragraph{In-context learning.}
Empirical studies of ICL show that demonstrations provide more than a small supervised training set. A complementary theoretical line asks which algorithm a transformer executes in-context. linear regression studies show that transformers can implement or approximate least-squares-style predictors in controlled settings \cite{garg2022what,akyurek2023what}. Bayesian and algorithm-selection views explain ICL as posterior inference or as selecting among latent algorithms from the prompt \cite{xie2022explanation,bai2023transformers}. Gradient-descent and meta-optimizer interpretations analyze how transformer computations can emulate iterative optimization procedures \cite{vonoswald2023transformers,dai2023why,ahn2023transformers}. Other analyses study how pretraining task diversity, meta-learning objectives, and learnability conditions shape the emergence of ICL \cite{raventos2023pretraining,kirsch2022general,wies2023learnability}. Mechanistic work further identifies induction-head circuits and task-vector phenomena as compact carriers of prompt-induced behavior \cite{olsson2022incontext,yang2025taskvectors}. These works support a task-inference view of ICL. Our focus is different: we ask when the inferred task should be treated as a low-dimensional concept coordinate rather than as an arbitrary ambient parameter. 

\paragraph{Concept learning.}
Concept-based methods seek representations in which semantically meaningful variables can be inspected, constrained, or intervened on~\cite{koh2020concept, Hu2025ECBM, Hu2025SSCBM}. Post-hoc concept directions and whitening methods identify or align internal features with human-interpretable concepts \cite{kim2018tcav,chen2020concept}. Ante-hoc architectures make concept variables explicit through bottlenecks or concept embeddings \cite{koh2020concept,zarlenga2022concept}. However, unsupervised latent factors are generally unidentifiable without structural assumptions, weak supervision, or an appropriate invariance class \cite{locatello2019challenging,khemakhem2020variational}. Recent analyses also caution that concept bottlenecks may leak non-concept information rather than enforcing a clean semantic factorization \cite{almudevar2026bottleneck}. In contrast, our concept variable is not a named human label and is not imposed during model training. It is a latent task coordinate inferred online from demonstrations.

\paragraph{Activation patching.}
Activation patching and interchange interventions test whether an internal variable causally mediates behavior by replacing activations between counterfactual runs. This logic underlies causal mediation and causal abstraction analyses \cite{vig2020investigating,geiger2021causal}. Circuit-level work studies how transformer components implement specific behaviors and how such circuits can be discovered automatically \cite{wang2023interpretability,conmy2023towards}. Patching conclusions depend on corruption schemes, metrics, intervention granularity, and controls; in particular, subspace patching can steer behavior without faithfully identifying the subspace used by the original computation \cite{zhang2024towards,makelov2024subspace}. Our intervention target is not an entire residual stream or an arbitrary direction found by search, but the low-dimensional activation subspace predicted by the concept-subspace theory and tested with matched controls.

Overall, prior work largely treats ICL, concept-level representation, and causal intervention as separate phenomena. We connect them through a theory-first latent concept subspace that links statistical task structure, activation geometry, and causal behavior.

\section{Concept-Subspace In-Context Learning}
\label{sec:concept_subspace_icl}
\subsection{Problem Formalization}
We study a linear regression in-context learning (ICL) model with a shared low-dimensional task structure, following the stylized linear regression settings widely used to analyze ICL in transformers and linear attention models \cite{guo2024how,he2025incontext}. Let $d\ge 1$ and let $\Lambda\in\mathbb{R}^{d\times d}$ be positive definite. For a task parameter $w\in\mathbb{R}^d$, demonstrations are sampled as
\begin{equation}
    x_i\stackrel{i.i.d.}{\sim}\mathcal{N}(0,\Lambda),\qquad y_i=\langle w,x_i\rangle,\qquad i=1,\dots,M.
\label{eq:demo_gen}
\end{equation}
A query $x\sim\mathcal{N}(0,\Lambda)$ is sampled independently, and the goal is to predict $y=\langle w,x\rangle$.

\begin{assumption}[Concept-subspace task family]
\label{ass:concept_subspace}
Let $U\in\mathbb{R}^{d\times r}$ have orthonormal columns, where $1\le r\le d$. The task family varies only in the concept subspace $\operatorname{span}(U)$: $w=U\beta$, with $\beta\sim\mathcal{N}(0,I_r)$. Equivalently, $y_i=\langle \beta,U^\top x_i\rangle$. We call $z:=U^\top x\in\mathbb{R}^r$ the concept coordinates of $x$.
\end{assumption}

Let $U_\perp\in\mathbb{R}^{d\times(d-r)}$ be any orthonormal complement of $U$, and write the covariance in the basis $[U\ U_\perp]$:
\begin{equation}
    \Lambda=\begin{bmatrix}\Lambda_{11}&\Lambda_{12}\\ \Lambda_{21}&\Lambda_{22}\end{bmatrix},\qquad \Lambda_{11}=U^\top\Lambda U,\quad \Lambda_{22}=U_\perp^\top\Lambda U_\perp,\quad \Lambda_{12}=U^\top\Lambda U_\perp.
\end{equation}
We use two covariance regimes: \textbf{(BD)} exact block diagonality, $\Lambda_{12}=0$; and \textbf{(NBD)} near block diagonality, $\|\Lambda_{12}\|_{\mathrm{op}}\le\rho$. Under (BD), concept and nuisance coordinates are independent because the joint distribution is Gaussian.

\subsection{Theoretical Analysis}
For a demonstration set $\mathcal{D}:=\{(x_i,y_i)\}_{i=1}^M$, we focus on the ridge regression proxy for linear ICL:
\begin{equation}
    \hat w_{\mathcal{D}}=(S+\lambda I_d)^{-1}\left(\frac1M\sum_{i=1}^M x_i y_i\right),\qquad f_{\mathcal{D}}(x)=\langle \hat w_{\mathcal{D}},x\rangle,\qquad S=\frac1M\sum_{i=1}^M x_ix_i^\top.
\label{eq:ridge_icl}
\end{equation}
Ridge and least-squares estimators are standard statistical proxies for test-time regression \cite{hoerl1970ridge}; in the ICL literature, least-squares, ridge, gradient-descent, and preconditioned-gradient-descent predictors arise as explicit or implicit algorithms implemented by trained transformers. All ambient-space statements below assume $\lambda>0$, so that $S+\lambda I_d$ is invertible. The unregularized least-squares case can be recovered in concept space under the usual full-rank condition on $Z^\top Z$, or in ambient space by replacing inverses with Moore--Penrose pseudoinverses under additional rank assumptions.

\begin{lemma}[Exact subspace decomposition]
\label{lem:subspace_reduction}
Assume $w=U\beta$ and $\lambda>0$. Let $Z\in\mathbb{R}^{M\times r}$ have rows $z_i^\top=(U^\top x_i)^\top$, let $y=(y_1,\dots,y_M)^\top$, and write $x=Uz+U_\perp t$, where $z=U^\top x$ and $t=U_\perp^\top x$. Define the empirical blocks $S_{11}=U^\top S U$, $S_{12}=U^\top S U_\perp$, $S_{21}=S_{12}^\top$, $S_{22}=U_\perp^\top S U_\perp$, and set $A=S_{11}+\lambda I_r$, $B=S_{12}$, $D=S_{22}+\lambda I_{d-r}$, $H=D-S_{21}A^{-1}B$. Then $H\succ0$, and
\begin{equation}
f_{\mathcal{D}}(x)=z^\top\hat\beta_{\mathcal{D}}+\Delta_{\mathcal{D}}(x),\qquad \hat\beta_{\mathcal{D}}=A^{-1}S_{11}\beta=(Z^\top Z+M\lambda I_r)^{-1}Z^\top y,
\label{eq:subspace_decomp}
\end{equation}
with
\begin{equation}
\Delta_{\mathcal{D}}(x)=\big(t^\top-z^\top A^{-1}B\big)\gamma_{\mathcal{D}},\qquad \gamma_{\mathcal{D}}=H^{-1}S_{21}(\beta-\hat\beta_{\mathcal{D}}).
\label{eq:delta_correct}
\end{equation}
Thus the off-subspace contribution is controlled by the bottom coordinate $\gamma_{\mathcal{D}}=U_\perp^\top\hat w_{\mathcal{D}}$, which is induced through empirical cross-subspace coupling. In particular, under (BD), the population cross-covariance is zero, while the finite-sample term $S_{21}$ remains.
\end{lemma}

For compactness, write $L_\delta:=r+\log(a_0/\delta)$, $R_\perp(\delta):=\sqrt{\operatorname{tr}(\Lambda_{22})}+\sqrt{\|\Lambda_{22}\|_{\mathrm{op}}\log(a_0/\delta)}$, and $\eta_M(\delta):=R_\perp(\delta)\sqrt{L_\delta/M}$, where $a_0>0$ is a universal numerical constant.

\begin{theorem}[Concept-coordinate estimation and off-subspace leakage]
\label{thm:concept_dimension}
Under the data model in Equation~\eqref{eq:demo_gen} and the concept-subspace model in Assumption~\ref{ass:concept_subspace}, take $\lambda=\lambda_0/M$ for a fixed $\lambda_0>0$. Under (BD), if $M\ge K_0 L_\delta$, for a constant $K_0$ depending only on $\Lambda_{11}$ and $\lambda_0$, then with probability at least $1-\delta$ over $\beta$, the demonstrations, and the query, the following hold. Let $z=U^\top x$, $t=U_\perp^\top x$, and let $\Delta_{\mathcal{D}}(x)$ be the off-subspace term in Lemma~\ref{lem:subspace_reduction}. Then
\begin{equation}
\|\hat\beta_{\mathcal{D}}-\beta\|\le K_\beta\frac{\sqrt{L_\delta}}{M},\qquad |\Delta_{\mathcal{D}}(x)|\le K_\Delta\left(\|t\|+\frac{R_\perp(\delta)}{\sqrt M}\|z\|\right)\eta_M(\delta).
\label{eq:concept_and_leakage_rates}
\end{equation}
Consequently, the prediction error decomposes as
\begin{equation}
|f_{\mathcal{D}}(x)-\langle w,x\rangle|\le \underbrace{\|z\|\|\hat\beta_{\mathcal{D}}-\beta\|}_{\text{concept estimation error}}+\underbrace{|\Delta_{\mathcal{D}}(x)|}_{\text{off-space leakage}},
\label{eq:err_split_labeled}
\end{equation}
and in particular
\begin{equation}
|f_{\mathcal{D}}(x)-\langle w,x\rangle|\le \underbrace{K_\beta\|z\|\frac{\sqrt{L_\delta}}{M}}_{\text{concept estimation error}}+\underbrace{K_\Delta\left(\|t\|+\frac{R_\perp(\delta)}{\sqrt M}\|z\|\right)\eta_M(\delta)}_{\text{off-space leakage}}.
\label{eq:final_rate_compact}
\end{equation}
Moreover, for every $v\in\operatorname{span}(U)^\perp$,
\begin{equation}
|f_{\mathcal{D}}(x+v)-f_{\mathcal{D}}(x)|\le K_\Delta\eta_M(\delta)\|v\|.
\label{eq:invariance_compact}
\end{equation}
Under (NBD), the same bounds hold after replacing $\eta_M(\delta)$ by $\eta_M(\delta)+K_\rho\rho\sqrt{L_\delta}$ and $R_\perp(\delta)/\sqrt M$ by $R_\perp(\delta)/\sqrt M+K_\rho'\rho$, where $K_\rho,K_\rho'$ depend only on $\Lambda_{11}$ and $\lambda_0$.
\end{theorem}

\begin{remark}[Interpretation of the dimension dependence]
The leading concept-estimation term in Equation~\eqref{eq:concept_and_leakage_rates} depends on $r$, not on the ambient dimension $d$. This aligns with representation-based ICL analyses in which the effective predictor is linear only after projecting to a learned or task-relevant representation. The ambient predictor can still be sensitive to nuisance directions through the explicit cross-subspace coefficient $\gamma_{\mathcal{D}}$. This dependence is unavoidable for ambient ridge regression unless one controls the nuisance covariance scale, e.g., through whitening, bounded effective nuisance energy, or an explicit bound on $R_\perp(\delta)$. Thus the theorem should be read as a decomposition result: concept inference is $r$-dimensional, while residual nuisance sensitivity is governed by empirical and population cross-subspace coupling.
\end{remark}

\subsection{Statistical and Bayesian Consequences}
Theorem~\ref{thm:concept_dimension} shows that the intrinsic estimation problem is $r$-dimensional and that deviations from exact nuisance invariance are controlled by cross-subspace coupling. We record two consequences.

\paragraph{(i) Bayesian concept inference.}
\begin{corollary}[Posterior interpretation]
\label{cor:bayes_concept}
Assume the noisy linear-Gaussian concept model $y=Z\beta+\varepsilon$, where $\beta\sim\mathcal{N}(0,I_r)$ and $\varepsilon\sim\mathcal{N}(0,\sigma^2 I_M)$. Then
\begin{equation}
    \beta\mid\mathcal{D}\sim\mathcal{N}(\mu_{\mathcal{D}},\Sigma_{\mathcal{D}}),\qquad \Sigma_{\mathcal{D}}=\left(I_r+\frac1{\sigma^2}Z^\top Z\right)^{-1},\qquad \mu_{\mathcal{D}}=\Sigma_{\mathcal{D}}\frac1{\sigma^2}Z^\top y.
\end{equation}
The Bayes predictor is $\mathbb{E}[y\mid x,\mathcal{D}]=\mu_{\mathcal{D}}^\top U^\top x$, and the posterior mean coincides with the concept-space ridge estimator in Lemma~\ref{lem:subspace_reduction} when $\lambda=\sigma^2/M$. This is the standard conjugate Bayesian linear regression posterior \cite{bishop2006pattern,murphy2012machine,yang2024promipl}, specialized to the concept coordinate space.
\end{corollary}

This posterior lives in the concept coordinate space rather than in the ambient parameter space.

\paragraph{(ii) Stability under nuisance perturbations.}
\begin{corollary}[Nuisance stability]
\label{cor:ood_stability}
Under the assumptions of Theorem~\ref{thm:concept_dimension} and under (BD), with probability at least $1-\delta$, for every $v\in\operatorname{span}(U)^\perp$,
\begin{equation}
    |f_{\mathcal{D}}(x+v)-f_{\mathcal{D}}(x)|\le K_{\mathrm{stab}}R_\perp(\delta)\|v\|\sqrt{\frac{r+\log(a_0/\delta)}{M}}.
\end{equation}
For the Bayes concept predictor $x\mapsto\mu_{\mathcal{D}}^\top U^\top x$, the same perturbation has exactly zero effect: $\mu_{\mathcal{D}}^\top U^\top(x+v)=\mu_{\mathcal{D}}^\top U^\top x$.
\end{corollary}

\paragraph{Noisy labels.}
For concept-space ridge or Bayesian inference, homoscedastic label noise adds the usual $\sigma\sqrt{r/M}$ estimation term, as in standard finite-sample analyses of linear regression \cite{hastie2009elements,wainwright2019high}. The ambient ridge predictor satisfies the same exact block decomposition as Lemma~\ref{lem:subspace_reduction}, but with an additional off-subspace noise term. Therefore, any ambient-space noisy-label stability bound must control this term explicitly. The full derivation is given in Appendix~\ref{app:noisy}.

\subsection{Identifiability of the Concept Subspace}
We next formalize when the concept subspace is statistically recoverable. This question is related to the broader issue that latent concepts are generally identifiable only under additional structure or auxiliary information \cite{hendel2023incontext,liu2024incontext}. Consider $T$ tasks with coefficients $\beta_t\in\mathbb{R}^r$, observations $x_{ti}\sim\mathcal{N}(0,\Lambda)$, and labels $y_{ti}=\langle U\beta_t,x_{ti}\rangle+\varepsilon_{ti}$, where $\mathbb{E}[\varepsilon_{ti}\mid x_{ti},t]=0$.

\begin{definition}[Subspace identifiability up to rotation]
\label{def:identifiability}
The concept subspace is identifiable up to rotation if any alternative orthonormal representation $\tilde U\in\mathbb{R}^{d\times r}$ that induces the same task-conditioned linear predictors for all tasks satisfies $\operatorname{span}(\tilde U)=\operatorname{span}(U)$. In that case, the coordinate systems differ only by an orthogonal change of basis: $\tilde U=UR$ and $\tilde\beta_t=R^\top\beta_t$ for some $R\in\mathbb{R}^{r\times r}$ with $R^\top R=I_r$. 
\end{definition}

\begin{proposition}[Task-conditioned identifiability]
\label{prop:identifiability}
Assume $\Lambda\succ0$ is known and $\mathbb{E}[x\varepsilon\mid t]=0$. Let $B=[\beta_1,\dots,\beta_T]\in\mathbb{R}^{r\times T}$ have rank $r$. For each task define $m_t:=\mathbb{E}[xy\mid t]\in\mathbb{R}^d$ and $G_{xy}:=[m_1,\dots,m_T]$. Then
\begin{equation}
    m_t=\Lambda U\beta_t,\qquad G_{xy}=\Lambda U B,\qquad \operatorname{rank}(G_{xy})=r,
\end{equation}
and therefore
\begin{equation}
    \operatorname{span}(\Lambda^{-1}G_{xy})=\operatorname{span}(U).
\end{equation}
If $\beta_t\stackrel{i.i.d.}{\sim}\mathcal{N}(0,I_r)$ and $T\ge r$, then $\operatorname{rank}(B)=r$ almost surely. Hence the concept subspace is identifiable up to rotation from task-conditioned first moments. In contrast, the unconditional first moment satisfies $\mathbb{E}[xy]=\Lambda U\mathbb{E}[\beta]=0$ under the zero-mean task prior and cannot identify $U$ by itself. \looseness=-1
\end{proposition}

This justifies estimating a concept subspace $\widehat U$ from multi-task data, probes, or CCA-type objectives: the statistically meaningful target is the subspace, not a particular basis for its coordinates.

\section{From Concept-Subspace ICL to Mechanistic Predictions}
\label{sec:concept_mechanistic}
The preceding section is a functional statistical analysis. It shows that, in the linear concept-subspace model, task inference can be carried out in the low-dimensional coordinate $\beta$, while ambient nuisance sensitivity is mediated by cross-subspace leakage. This does not prove how a transformer implements ICL internally. It motivates a mechanistic hypothesis: if a transformer approximates concept-subspace inference, then task-relevant information should be concentrated in a low-dimensional, task-aligned activation subspace~\cite{todd2024function,merullo2024language}. This hypothesis is consistent with work on task vectors, function vectors, in-context vectors, and concept-based vector arithmetic in ICL \cite{dong2026understanding,bu2025provable}. \looseness=-1

\paragraph{Concept causal effect.}
If ICL proceeds by estimating a low-dimensional concept coordinate $\hat\beta(\mathcal{P})$ from the prompt $\mathcal{P}$, then one expects some internal representation to carry information about this estimate. The relevant empirical question is whether there exists a low-dimensional activation subspace whose state is causally sufficient for a substantial part of the in-context prediction. We operationalize this through a \emph{Concept Causal Effect} (CCE): starting from a corrupted prompt, where concept consistency is broken, we patch only a candidate concept subspace from the clean run. Activation patching and causal tracing are standard tools for such interventions \cite{meng2022locating}, but subspace interventions require careful controls because naive subspace patching can produce interpretability illusions \cite{makelov2024subspace}. Recovery of the clean behavior indicates that the patched subspace carries task information used by the model; failure of complementary-space or matched-rank controls argues against a purely intervention-size explanation.

\paragraph{Concept mediation hypothesis.}
The theory motivates, but does not assume, the following causal abstraction. Suppose that at some layer $\ell^\star$, the model state admits a decomposition $s(\mathcal{P})=T\hat\beta(\mathcal{P})+\text{nuisance}$, where the nuisance component has little direct effect on the output once $\hat\beta(\mathcal{P})$ is fixed. Then patching the component $T\hat\beta(\mathcal{P})$ should restore the model's prediction under corrupted prompts, whereas patching orthogonal directions should not. We treat this as a falsifiable mechanistic hypothesis rather than as a theorem about arbitrary transformers.

\paragraph{Testable predictions.}
The framework leads to three empirical predictions: (i) a small number of task-aligned activation directions should recover a substantial fraction of ICL behavior; (ii) patching the orthogonal complement or same-rank unrelated subspaces should not recover comparable behavior; and (iii) injecting the concept subspace from a different task should redirect predictions toward that task. The experiments below test these predictions directly, without claiming full circuit recovery. \looseness=-1

\section{Mechanistic Validation for Concept-Subspace ICL}
In this section, we test whether the theory-motivated concept subspace appears as a causally effective activation subspace in large language models. The goal is not full circuit recovery. Instead, we ask whether a compact, task-aligned residual-stream subspace carries a substantial fraction of the recoverable ICL signal under controlled interventions.

\subsection{Experimental Protocol and Metrics}
\noindent \textbf{Tasks.}
We evaluate structured multi-relation tasks from CounterFact~\cite{meng2022locating}. Retaining 34 mapping relations (e.g., country $\rightarrow$ capital), we evaluate 90 held-out queries per relation using 12-shot prompts. For each query, we contrast two conditions: a \textit{clean context} governed by a consistent concept and a \textit{corrupted context} that destroys concept consistency (e.g., via label shuffling) while preserving prompt format, length, and token statistics. This controlled corruption isolates concept mediation from superficial pattern matching. Moreover, Appendix~\ref{sec:crosslingual} evaluates implicit cross-lingual translation tasks, confirming that this mediation extends from factual relations to abstract rules.

\noindent \textbf{Evaluation Settings.}
We compare five primary conditions: (i) the corrupted baseline, (ii) the clean few-shot upper bound, (iii) full clean-to-corrupted patching, (iv) concept-subspace-only patching, and (v) complementary-subspace patching. We also evaluate two same-rank controls: a random orthogonal subspace and a cross-task subspace extracted from unrelated relations.

\noindent \textbf{Evaluation Metrics.}
We report top-1 accuracy for standard evaluation. For patching experiments, we additionally report the \emph{recovery rate}, i.e., the fraction of the clean--corrupted accuracy gap restored by the intervention: $\mathrm{Recovery~Rate}(\%) = 100 \cdot
\frac{\mathrm{Acc}_{\mathrm{patch}} - \mathrm{Acc}_{\mathrm{corr}}}
{\mathrm{Acc}_{\mathrm{clean}} - \mathrm{Acc}_{\mathrm{corr}}}$. For concept swapping and same-rank controls, we report \emph{override success}, i.e., the fraction of examples whose final prediction follows the injected target relation. More implementation details are provided in Appendix~\ref{subsec:implementation}.

\begin{figure*}[t]
  \centering
  \includegraphics[width=0.49\textwidth]{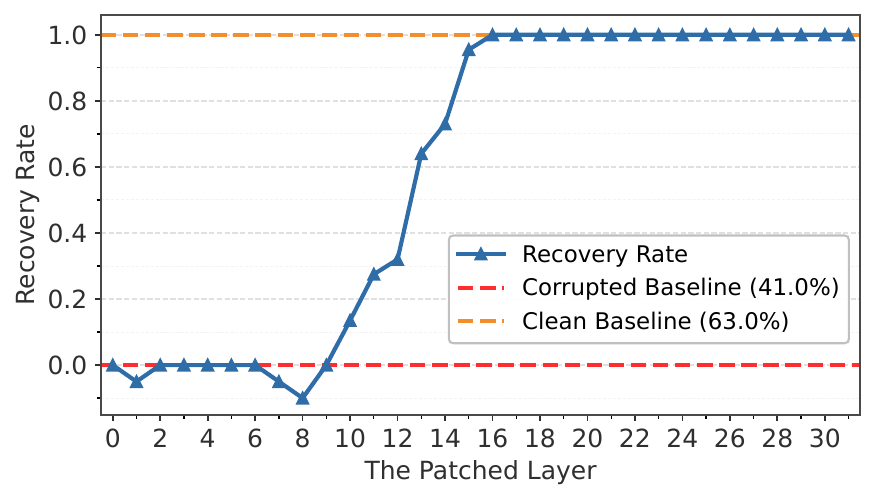}
  \hfill
  \includegraphics[width=0.49\textwidth]{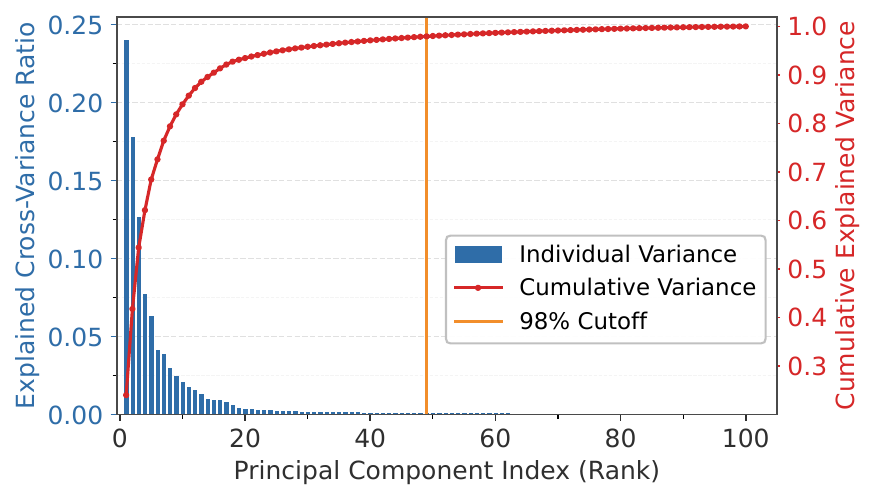}
  \caption{
    Layerwise localization and low-rank extraction of task-aligned residual directions (Llama-3-8B and CounterFact). \emph{Left}: clean-to-corrupted patching localizes layers where task information is causally accessible. \emph{Right}: the 98\% explained-cross-variance threshold selects a low-dimensional concept subspace at layer 30.
  }
  \label{fig:layerwise_and_scree}
\end{figure*}

\subsection{Localizing and Estimating the Concept Subspace}
\paragraph{Localizing the concept subspace.} 
We first localize where task-relevant information becomes available for in-context prediction. For each layer $\ell$, we run the Llama-3-8B model on clean and corrupted prompts, patch the clean query-token residual state $r_{\ell}^{\mathrm{clean}}$ into the corrupted run, and measure the resulting recovery. Figure~\ref{fig:layerwise_and_scree} (left) reveals a strict late-layer emergence of recovery: patching effects remain near zero through layers 0--14, become visible around layers 15--20, and rise steeply after layer 21. This suggests that task-level information is not uniformly distributed, but consolidates only after deep contextual integration. Consequently, we select layer 30 ($\ell^\star=30$) for all subsequent concept-subspace estimation and causal patching experiments.
Further layerwise analyses confirm the late-layer emergence of this causal concept bottleneck (see Appendix~\ref{appendix:layerwise_analysis}). This design choice is methodologically important: a linear subspace estimator is only meaningful once the task variable has become explicitly represented in the residual stream. We further discuss the results on Qwen2.5-7B~\cite{qwen_qwen25_2024} in Appendix~\ref{sec:qwen_results} and on the cross-lingual task~\cite{lample2018word} in Appendix \ref{sec:crosslingual}.
\begin{figure*}[t]
    \centering
    \includegraphics[width=0.49\textwidth]{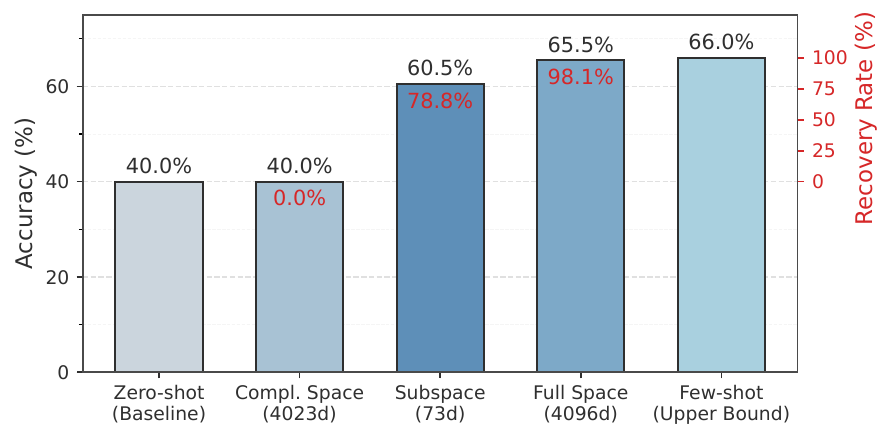}
    \hfill
    \includegraphics[width=0.49\textwidth]{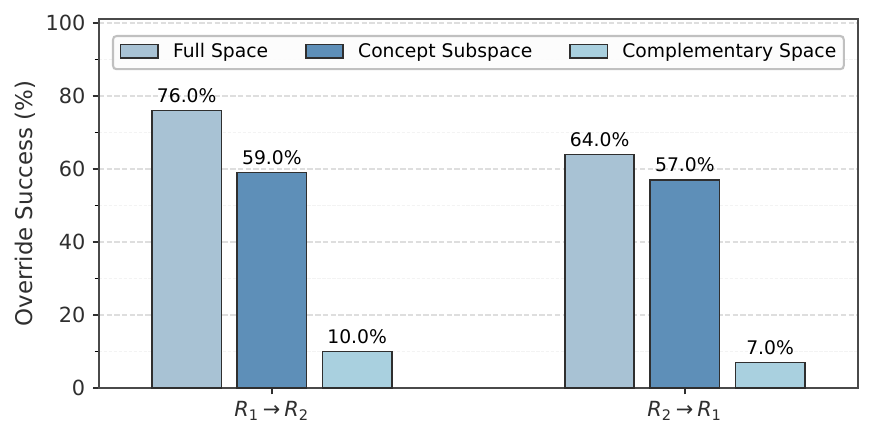}
    \caption{
    Concept-subspace patching and swapping. \emph{Left:} the learned subspace recovers most of the clean-corrupted gap. \emph{Right:} subspace swaps redirect predictions toward the injected relation.
    }
    \label{fig:main_interventions}
\end{figure*}

\paragraph{Estimating the concept subspace.} 
Guided by earlier localization, we estimate the task-aligned subspace at layer $\ell^\star=30$ using clean 12-shot prompts across 34 structured CounterFact relations (e.g., country $\rightarrow$ capital, entity $\rightarrow$ language), with 90 held-out queries per relation. To isolate relation-level directions from entity-specific variations, we take the query-token residual activations $H \in \mathbb{R}^{n \times d}$ and the aligned target-supervision matrix $Y$ to compute their cross-covariance $C_{HY}=\frac{1}{n}H^\top Y$. Performing SVD ($C_{HY}=U\Sigma V^\top$), we define the subspace by selecting the minimal rank $r$ that captures $\ge 98\%$ of the cumulative cross-variance: $\frac{\sum_{i=1}^{r}\sigma_i^2}{\sum_i \sigma_i^2} \ge 0.98$.
We then define the empirical concept subspace as $\widehat{U}_{\ell^\star}=U_{[:,1:r]}$. Equivalently, if one first fits a linear probe from $H$ to $Y$, $\widehat{U}_{\ell^\star}$ is given by the leading left-singular directions of the probe weight matrix. This construction is important because it suppresses entity-specific redundancy while preserving only those activation directions that are consistently associated with relation-level output structure across prompts.
Figure~\ref{fig:layerwise_and_scree} (right) reveals a sharply concentrated spectrum at layer~30. A 98\% explained-cross-variance threshold selects only a minimal fraction of directions in the representative run, and the cutoff remains in a narrow range across runs (see Appendix~\ref{subsec:stability_analyses} for evidence of the estimated rank's stability across varying estimation sizes and few-shot counts). 
This low rank relative to the 4096-dimensional stream suggests that ICL compresses demonstrations into a compact task representation rather than a diffuse code. While this spectral signature is purely correlational, it enables our causal test: if the concept-subspace hypothesis holds, intervening solely on $\widehat{U}_{\ell^\star}$ should recover most corrupted ICL behavior, whereas intervening on its orthogonal complement should prove ineffective.

\subsection{Concept-Subspace Patching Recovers ICL Behavior}
Let $P$ denote the orthogonal projector onto $\widehat U_{\ell^\star}$. Given clean and corrupted runs for the same query, we intervene on the corrupted residual stream by patching only the learned concept component $\tilde{r}_{\ell^\star} = r_{\ell^\star}^{\text{corr}} + P\left(r_{\ell^\star}^{\text{clean}} - r_{\ell^\star}^{\text{corr}} \right)$. We compare this intervention with full-state patching, complementary-space patching using $I-P$, and the unpatched zero-shot baseline.  
Figure~\ref{fig:main_interventions} (left) shows a strong asymmetry between the learned subspace and its orthogonal complement. The zero-shot baseline obtains $40.0\%$ accuracy, while the clean few-shot upper bound reaches $66.0\%$. Full activation patching raises accuracy to $65.5\%$, recovering $98.1\%$ of the clean--corrupted gap. Patching only the 73-dimensional concept subspace still reaches $60.5\%$, corresponding to $78.8\%$ recovery, despite modifying only $73/4096 \approx 1.8\%$ of the dimensions. In contrast, patching the 4023-dimensional complement leaves accuracy unchanged at $40.0\%$, yielding $0.0\%$ recovery.
Thus, the intervention effect is not proportional to the number of patched dimensions. Most recoverable task-level signal is concentrated in a small task-aligned subspace, while the high-dimensional complement is insufficient to restore behavior. Directional noise perturbation experiments confirm that this low-dimensional concept subspace acts as the central causal bottleneck governing the model's logical inference (see Figure~\ref{fig:robustness_to_noise_few_shot}). Additionally, the concept subspace demonstrates higher stability against few-shot label noise than the full space (Appendix~\ref{appendix:demo_stability}). The result supports a more precise claim: the learned subspace captures the dominant recoverable causal component of the ICL computation in these structured relation tasks (see Appendix~\ref{appendix:scaling} for scaling analyses confirming this pattern across 1 to 12 demonstrations).

\subsection{Concept Swaps and Matched-Rank Specificity Controls}
\paragraph{Concept swaps.}
To stringently test causal mediation via behavioral redirection, we extract clean layer-30 residual states for entities participating in two distinct relations ($R_1, R_2$) with corresponding concept parameters ($\beta_1, \beta_2$). During a source-relation forward pass, we inject target-relation information using three parallel interventions: (i) \emph{full-space replacement} (swapping the entire residual state); (ii) \emph{concept-subspace replacement} (swapping solely the $\widehat{U}_{\ell^\star}$ projection); and (iii) \emph{complementary-space replacement} (swapping only the orthogonal complement while retaining the source projection). We evaluate \emph{override success}—the fraction of predictions successfully flipped to the target relation. Systematic logical shifts induced specifically by subspace injection would firmly establish this low-dimensional representation as a causal mediator of in-context learning.

Figure~\ref{fig:main_interventions} (right) shows that directional control is concentrated in the learned subspace. Full-space replacement achieves $76.0\%$ for $R_1\!\rightarrow\!R_2$ and $64.0\%$ for $R_2\!\rightarrow\!R_1$, close to the native few-shot accuracies of the target relations ($75.0\%$ for $R_2$, $64.0\%$ for $R_1$). Replacing only the concept subspace still yields substantial override, whereas replacing the complementary space is largely ineffective. Therefore, manipulating just the minimal concept subspace is sufficient to invert the model's logical outputs, capturing the vast majority of the adversarial override achieved by full-space replacement. The inertness of the complementary space supports that the learned subspace as the dominant causal bottleneck governing in-context task execution.

\begin{table}[ht]
\centering
\begin{minipage}[t]{0.67\textwidth}
\paragraph{Matched-rank specificity controls.}
We ask whether the overwrite effect is specific to the learned concept subspace or can be reproduced by an arbitrary low-rank intervention. Using the same source--target relation setup as in the concept-swap experiment, we patch the target hidden state into the source inference trace after projecting the update onto one of three matched-dimensional subspaces at the target layer: the learned concept subspace $U_{\mathrm{learned}}$, a random subspace $U_{\mathrm{rand}}$, and a cross-task subspace $U_{\mathrm{cross}}$ estimated from an unrelated task. We report the adversarial overwrite success rate. 
\end{minipage}
\hfill
\begin{minipage}[t]{0.3\textwidth}

\caption{Matched-rank controls for override success.}
\label{tab:matched_rank_controls}
\centering
\small
\begin{tabularx}{\linewidth}{Xc}
    \hline
    Intervention & Success (\%) \\
    \hline
    None  & 2.5 \\
    $U_{\mathrm{learned}}$ & 58.8 \\
    $U_{\mathrm{rand}}$ & 2.3 \\
    $U_{\mathrm{cross}}$ & 3.3 \\
    \hline
\end{tabularx}
\end{minipage}
\end{table}
Table~\ref{tab:matched_rank_controls} shows that only the learned subspace ($U_{\mathrm{learned}}$) achieves substantial override ($58.8\%$), while matched-rank controls ($U_{\mathrm{rand}}$, $U_{\mathrm{cross}}$) remain near the $2.5\%$ baseline ($\le 3.3\%$). Since dimensionality is controlled, this gap suggests that causal efficacy stems specifically from task-aligned directions rather than arbitrary low-rank edits. Alongside the inert complementary space, this establishes the learned concept subspace as a highly specific causal mediator of in-context behavior.

\begin{figure}[t]
  \centering
  \includegraphics[width=0.85\linewidth]{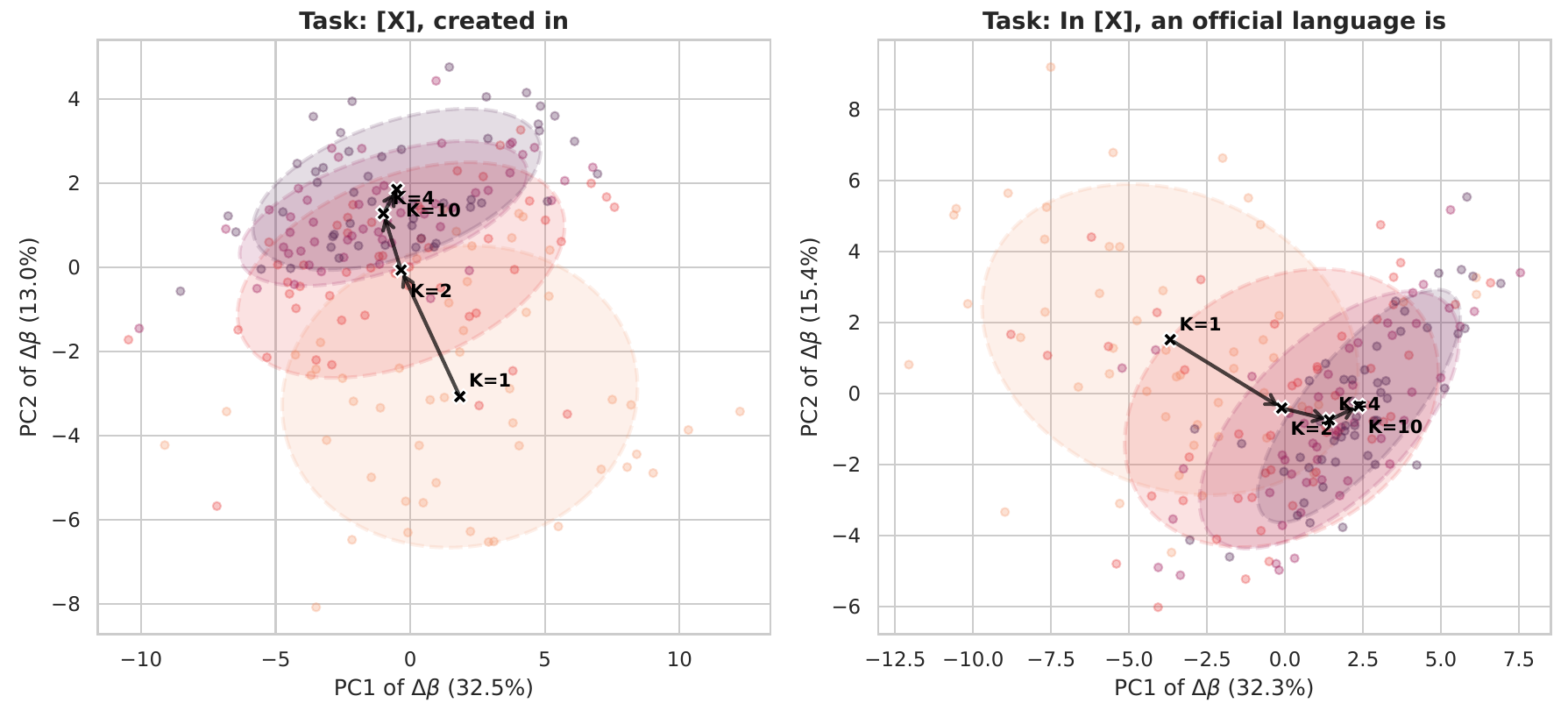}
  \caption{Few-shot contraction of debiased concept-coordinate estimates. Query-centered concept-coordinate displacements concentrate as the number of demonstrations increases.}
  \label{fig:debiased_fewshot_dynamics}
\end{figure}

\subsection{Few-Shot Contraction of Debiased Concept Coordinates}
To examine whether additional demonstrations reduce task-coordinate uncertainty, we isolate context-induced updates from query priors via debiased displacements. Specifically, at layer $\ell^\star$, the update for relation $r$, query $q$, and $K$-shot context $c$ is defined as $\Delta\beta_K^{(r,q,c)} = \beta_K^{(r,q,c)} - \beta_0^{(r,q)}$, where $\beta_0^{(r,q)}$ is the zero-shot baseline. Evaluating eight semantic relations (50 parallel contexts each, $K\in\{1,2,4,10\}$), we project $\Delta\beta_K$ onto its first two principal components, visualizing uncertainty dynamics through $K$-specific $1.5\sigma$ covariance ellipses. Figure~\ref{fig:debiased_fewshot_dynamics} shows two representative relations. At $K=1$, the displacement clouds have large covariance ellipses, indicating that a single demonstration yields a noisy, exemplar-dependent estimate. As $K$ increases, the clouds contract and their centroids stabilize, and the most visible contraction occurs from $K=1$ to $K=4$, with smaller changes from $K=4$ to $K=10$. The contraction pattern is relation-dependent: the ``created in'' relation exhibits a larger centroid shift, whereas the ``official language'' relation starts with a broader $K=1$ cloud and rapidly concentrates near the higher-shot centroid.
These trajectories support interpreting $\beta$-coordinates as few-shot estimates of task rules: after removing query-specific bias, the inferred coordinate becomes less sensitive to the sampled context as more demonstrations are provided.

\section{Conclusion}
\label{sec:conclusion}
We studied in-context learning through the lens of concept subspace learning. In the proposed linear model, structured task variation is represented by low-dimensional coordinates rather than by an unconstrained ambient predictor. Ridge and least-squares ICL proxies decompose into concept-coordinate regression and off-subspace leakage: the intrinsic task estimate is governed by the concept dimension under explicit covariance assumptions, while ambient nuisance effects enter through identifiable leakage terms.
This decomposition motivates a mechanistic prediction for transformers: in structured task families, recoverable task information should be concentrated in a compact, task-aligned activation subspace. Our experiments support this prediction on CounterFact-derived relation tasks. A low-dimensional residual-stream subspace recovers most of the clean--corrupted ICL gap, while complementary and matched-rank control subspaces have little effect. Concept-swap interventions further show that the same subspace can redirect predictions toward an injected relation. Additional experiments on another model and on a cross-lingual rule task show the same qualitative pattern.
These findings suggest that concept subspaces provide a useful bridge between statistical accounts of ICL and causal activation-space analyses. They do not imply that full ICL circuits are low-dimensional, nor that a single subspace explains all prompting behavior. Accordingly, broader evaluation across task families, model scales, languages, prompting regimes, and subspace estimators remains necessary to determine the scope of concept-subspace mediation.

\bibliographystyle{unsrtnat}
\bibliography{main}

\newpage
\appendix

\section{Additional Theory and Proofs}
\label{app:theory}
\subsection{Noisy labels in concept space}
\label{app:noisy}
We extend Equation~\eqref{eq:demo_gen} to homoscedastic noise:
\begin{equation}
    y_i=\langle U\beta,x_i\rangle+\varepsilon_i,\qquad \varepsilon_i\stackrel{i.i.d.}{\sim}\mathcal{N}(0,\sigma^2),\qquad \beta\sim\mathcal{N}(0,I_r).
\label{eq:noisy_demo}
\end{equation}
Let $Z\in\mathbb{R}^{M\times r}$ have rows $z_i^\top=(U^\top x_i)^\top$, let $y=(y_1,\dots,y_M)^\top$, and define the concept-space ridge estimator $\hat\beta_{\mathcal{D}}=(Z^\top Z+M\lambda I_r)^{-1}Z^\top y$.

\begin{corollary}[Noisy concept inference]
\label{cor:noisy_concept_inference}
Assume Equation~\eqref{eq:noisy_demo}, $\Lambda_{11}\succ0$, and $\lambda=\lambda_0/M$ for fixed $\lambda_0>0$. Let $L_\delta=r+\log(a_0/\delta)$. If $M\ge K_0 L_\delta$, then with probability at least $1-\delta$,
\begin{equation}
    \|\hat\beta_{\mathcal{D}}-\beta\|\le K_1\left(\frac{\sqrt{L_\delta}}{M}+\sigma\sqrt{\frac{L_\delta}{M}}\right).
\label{eq:noisy_full_beta}
\end{equation}
Consequently, for the concept predictor $\tilde f_{\mathcal{D}}(x):=(U^\top x)^\top\hat\beta_{\mathcal{D}}$,
\begin{equation}
    |\tilde f_{\mathcal{D}}(x)-\langle U\beta,x\rangle|\le \|U^\top x\|K_1\left(\frac{\sqrt{L_\delta}}{M}+\sigma\sqrt{\frac{L_\delta}{M}}\right).
\label{eq:noisy_full_pred}
\end{equation}
Moreover, $\tilde f_{\mathcal{D}}(x+v)=\tilde f_{\mathcal{D}}(x)$ for every $v\in\operatorname{span}(U)^\perp$.
\end{corollary}

\begin{remark}[Ambient ridge with noisy labels]
For the ambient ridge predictor $\hat w_{\mathcal{D}}=(S+\lambda I_d)^{-1}M^{-1}\sum_i x_i y_i$, the same block decomposition as Lemma~\ref{lem:subspace_reduction} holds, but the bottom coordinate obtains an additional off-subspace noise term. This term can depend on the nuisance covariance scale and regularization. Therefore, the clean $\sigma\sqrt{r/M}$ rate in Corollary~\ref{cor:noisy_concept_inference} should be interpreted as a concept-space or Bayesian inference statement unless additional assumptions control off-subspace noise amplification.
\end{remark}

\subsection{Auxiliary concentration facts}
We use two standard Gaussian concentration facts for non-asymptotic covariance and Gaussian tail bounds. First, if $z_i\stackrel{i.i.d.}{\sim}\mathcal{N}(0,\Lambda_{11})$ and $S_{11}=M^{-1}Z^\top Z$, then for $M\ge K(r+\log(a_0/\delta))$,
\begin{equation}
    S_{11}\succeq \frac12\lambda_{\min}(\Lambda_{11})I_r
\label{eq:aux_cov_lower}
\end{equation}
with probability at least $1-\delta$, where $K$ depends only on $\Lambda_{11}$. Second, if $g\sim\mathcal{N}(0,\Sigma)$, then with probability at least $1-\delta$,
\begin{equation}
    \|g\|\le \sqrt{\operatorname{tr}(\Sigma)}+\sqrt{2\|\Sigma\|_{\mathrm{op}}\log(1/\delta)}.
\label{eq:aux_gaussian_norm}
\end{equation}
Throughout the appendix, constants may change from line to line. The universal constant $a_0$ inside $L_\delta=r+\log(a_0/\delta)$ is chosen large enough to absorb union bounds.

\subsection{Proof of Lemma~\ref{lem:subspace_reduction}}
\begin{proof}
Let $Q=[U\ U_\perp]$. In the $Q$-basis, write $Q^\top\hat w_{\mathcal{D}}=(\alpha^\top,\gamma^\top)^\top$ and
\begin{equation}
    Q^\top S Q=\begin{bmatrix}S_{11}&S_{12}\\ S_{21}&S_{22}\end{bmatrix}.
\end{equation}
Since $w=U\beta$, the ridge normal equations are
\begin{equation}
    \begin{bmatrix}S_{11}+\lambda I_r&S_{12}\\ S_{21}&S_{22}+\lambda I_{d-r}\end{bmatrix}\begin{bmatrix}\alpha\\ \gamma\end{bmatrix}=\begin{bmatrix}S_{11}\beta\\ S_{21}\beta\end{bmatrix}.
\end{equation}
Set $A=S_{11}+\lambda I_r$, $B=S_{12}$, $D=S_{22}+\lambda I_{d-r}$, and $H=D-S_{21}A^{-1}B$. Since $S+\lambda I_d\succ0$, its Schur complement satisfies $H\succ0$. The first block equation gives $\alpha=A^{-1}S_{11}\beta-A^{-1}B\gamma=\hat\beta_{\mathcal{D}}-A^{-1}B\gamma$, where $\hat\beta_{\mathcal{D}}=A^{-1}S_{11}\beta$. Substituting this into the second block equation yields $H\gamma=S_{21}(\beta-\hat\beta_{\mathcal{D}})$, hence $\gamma=\gamma_{\mathcal{D}}=H^{-1}S_{21}(\beta-\hat\beta_{\mathcal{D}})$. Finally, for $x=Uz+U_\perp t$,
\begin{equation}
    f_{\mathcal{D}}(x)=z^\top\alpha+t^\top\gamma=z^\top\hat\beta_{\mathcal{D}}+\big(t^\top-z^\top A^{-1}B\big)\gamma_{\mathcal{D}}.
\end{equation}
The identity $\hat\beta_{\mathcal{D}}=(Z^\top Z+M\lambda I_r)^{-1}Z^\top y$ follows from $S_{11}=M^{-1}Z^\top Z$ and $y=Z\beta$.
\end{proof}

\subsection{Proof of Theorem~\ref{thm:concept_dimension}}
\begin{proof}
Let $\kappa_{11}:=\lambda_{\min}(\Lambda_{11})$. On the event in Equation~\eqref{eq:aux_cov_lower}, $S_{11}\succeq \kappa_{11}I_r/2$. Also, with probability at least $1-\delta/a_0$, $\|\beta\|\le K\sqrt{L_\delta}$. Since $\hat\beta_{\mathcal{D}}=(S_{11}+\lambda I_r)^{-1}S_{11}\beta$, we have
\begin{equation}
    \hat\beta_{\mathcal{D}}-\beta=-\lambda(S_{11}+\lambda I_r)^{-1}\beta.
\end{equation}
With $\lambda=\lambda_0/M$, this gives
\begin{equation}
    \|\hat\beta_{\mathcal{D}}-\beta\|\le \frac{\lambda_0/M}{\kappa_{11}/2+\lambda_0/M}\|\beta\|\le K_\beta\frac{\sqrt{L_\delta}}{M},
\end{equation}
which proves the concept-estimation bound.

We next control the off-subspace leakage. By Lemma~\ref{lem:subspace_reduction}, $\Delta_{\mathcal{D}}(x)=(t^\top-z^\top A^{-1}B)\gamma_{\mathcal{D}}$, where $\gamma_{\mathcal{D}}=H^{-1}S_{21}(\beta-\hat\beta_{\mathcal{D}})$. Since $S+\lambda I_d\succeq\lambda I_d$, the Schur complement satisfies $H\succeq\lambda I_{d-r}$, hence $\|H^{-1}\|_{\mathrm{op}}\le1/\lambda$. Combining this with $\beta-\hat\beta_{\mathcal{D}}=\lambda A^{-1}\beta$ gives
\begin{equation}
    \|\gamma_{\mathcal{D}}\|\le \|S_{21}A^{-1}\beta\|.
\end{equation}
Under (BD), conditional on $Z$ and $\beta$, the nuisance rows $t_i=U_\perp^\top x_i$ are independent of $z_i$ and distributed as $\mathcal{N}(0,\Lambda_{22})$. Therefore, for any fixed $a\in\mathbb{R}^r$,
\begin{equation}
    S_{21}a=\frac1M\sum_{i=1}^M t_i z_i^\top a\ \bigg|\ Z,a\sim\mathcal{N}\left(0,\frac{a^\top S_{11}a}{M}\Lambda_{22}\right).
\end{equation}
Taking $a=A^{-1}\beta$ and using $S_{11}\succeq\kappa_{11}I_r/2$, we obtain $a^\top S_{11}a\le K\|\beta\|^2$. Applying Equation~\eqref{eq:aux_gaussian_norm} conditionally and then the bound on $\|\beta\|$ gives
\begin{equation}
    \|\gamma_{\mathcal{D}}\|\le K R_\perp(\delta)\sqrt{\frac{L_\delta}{M}}=K\eta_M(\delta).
\label{eq:proof_gamma_bound}
\end{equation}
Similarly, conditional on $Z$ and the query concept coordinate $z$, applying the same argument with $a=A^{-1}z$ gives
\begin{equation}
    \|S_{21}A^{-1}z\|\le K R_\perp(\delta)\frac{\|z\|}{\sqrt M}.
\label{eq:proof_zeta_bound}
\end{equation}
Consequently,
\begin{equation}
    \begin{aligned}
    |\Delta_{\mathcal{D}}(x)|&\le \big(\|t\|+\|S_{21}A^{-1}z\|\big)\|\gamma_{\mathcal{D}}\|\\
    &\le K_\Delta\left(\|t\|+\frac{R_\perp(\delta)}{\sqrt M}\|z\|\right)\eta_M(\delta).
    \end{aligned}
\end{equation}
The prediction decomposition follows from
\begin{equation}
    f_{\mathcal{D}}(x)-\langle U\beta,x\rangle=z^\top(\hat\beta_{\mathcal{D}}-\beta)+\Delta_{\mathcal{D}}(x),
\end{equation}
and the final prediction bound follows by substituting the two estimates above.

It remains to prove invariance. If $v\in\operatorname{span}(U)^\perp$, then $U^\top(x+v)=U^\top x$ and $U_\perp^\top(x+v)=t+U_\perp^\top v$. The concept term and the $z^\top A^{-1}B\gamma_{\mathcal{D}}$ correction cancel, so
\begin{equation}
    f_{\mathcal{D}}(x+v)-f_{\mathcal{D}}(x)=\langle \gamma_{\mathcal{D}},U_\perp^\top v\rangle.
\end{equation}
Using Equation~\eqref{eq:proof_gamma_bound} gives $|f_{\mathcal{D}}(x+v)-f_{\mathcal{D}}(x)|\le K_\Delta\eta_M(\delta)\|v\|$.

For (NBD), use the Gaussian conditioning formula
\begin{equation}
    t_i=\Lambda_{21}\Lambda_{11}^{-1}z_i+\xi_i,\qquad \xi_i\sim\mathcal{N}(0,\Lambda_{22}-\Lambda_{21}\Lambda_{11}^{-1}\Lambda_{12}),
\end{equation}
where $\xi_i$ is independent of $z_i$. The residual part is bounded as in the BD case, since its covariance is dominated by $\Lambda_{22}$. The conditional mean contributes
\begin{equation}
    \|\Lambda_{21}\Lambda_{11}^{-1}S_{11}A^{-1}\beta\|\le K\rho\|\beta\|\le K_\rho\rho\sqrt{L_\delta}
\end{equation}
to $\|\gamma_{\mathcal{D}}\|$, and similarly contributes at most $K_\rho'\rho\|z\|$ to $\|S_{21}A^{-1}z\|$. This yields the stated replacements in the NBD bound.
\end{proof}

\subsection{Proof of Corollary~\ref{cor:noisy_concept_inference}}
\begin{proof}
Let $S_Z=M^{-1}Z^\top Z$. On the event $S_Z\succeq\lambda_{\min}(\Lambda_{11})I_r/2$, which holds with probability at least $1-\delta/a_0$ when $M\ge K_0L_\delta$, and on the event $\|\beta\|\le K\sqrt{L_\delta}$, write $y=Z\beta+\varepsilon$. Then
\begin{equation}
    \hat\beta_{\mathcal{D}}-\beta=-\lambda(S_Z+\lambda I_r)^{-1}\beta+(S_Z+\lambda I_r)^{-1}\frac1M Z^\top\varepsilon.
\end{equation}
For the bias term, $\lambda=\lambda_0/M$ gives
\begin{equation}
    \left\|\lambda(S_Z+\lambda I_r)^{-1}\beta\right\|\le K\frac{\sqrt{L_\delta}}{M}.
\end{equation}
For the variance term, condition on $Z$. Since $M^{-1}Z^\top\varepsilon\sim\mathcal{N}(0,\sigma^2S_Z/M)$, we may write
\begin{equation}
    (S_Z+\lambda I_r)^{-1}\frac1MZ^\top\varepsilon=(S_Z+\lambda I_r)^{-1}S_Z^{1/2}g,\qquad g\sim\mathcal{N}\left(0,\frac{\sigma^2}{M}I_r\right).
\end{equation}
On $S_Z\succeq\lambda_{\min}(\Lambda_{11})I_r/2$, the operator norm $\|(S_Z+\lambda I_r)^{-1}S_Z^{1/2}\|_{\mathrm{op}}$ is bounded by a constant depending only on $\Lambda_{11}$. A $\chi^2$ tail for $\|g\|$ therefore gives
\begin{equation}
    \left\|(S_Z+\lambda I_r)^{-1}\frac1MZ^\top\varepsilon\right\|\le K\sigma\sqrt{\frac{L_\delta}{M}}
\end{equation}
with probability at least $1-\delta/a_0$. Combining the bias and variance bounds proves Equation~\eqref{eq:noisy_full_beta}. The prediction bound follows from
\begin{equation}
    |\tilde f_{\mathcal{D}}(x)-\langle U\beta,x\rangle|=|z^\top(\hat\beta_{\mathcal{D}}-\beta)|\le \|z\|\|\hat\beta_{\mathcal{D}}-\beta\|.
\end{equation}
Finally, if $v\in\operatorname{span}(U)^\perp$, then $U^\top(x+v)=U^\top x$, hence $\tilde f_{\mathcal{D}}(x+v)=\tilde f_{\mathcal{D}}(x)$.
\end{proof}

\subsection{Noisy ambient decomposition}
For completeness, we record the exact ambient decomposition with noisy labels. Let $T_\perp\in\mathbb{R}^{M\times(d-r)}$ have rows $t_i^\top=(U_\perp^\top x_i)^\top$, and define $e_z=M^{-1}Z^\top\varepsilon$ and $e_t=M^{-1}T_\perp^\top\varepsilon$. In the basis $[U\ U_\perp]$, the normal equations are
\begin{equation}
    \begin{bmatrix}A&B\\ S_{21}&D\end{bmatrix}\begin{bmatrix}\alpha\\ \gamma\end{bmatrix}=\begin{bmatrix}S_{11}\beta+e_z\\ S_{21}\beta+e_t\end{bmatrix}.
\end{equation}
The concept-space noisy ridge estimator is $\hat\beta_{\mathcal{D}}=A^{-1}(S_{11}\beta+e_z)$. Solving the block equations gives
\begin{equation}
    \alpha=\hat\beta_{\mathcal{D}}-A^{-1}B\gamma_{\mathcal{D}},\qquad \gamma_{\mathcal{D}}=H^{-1}\big(S_{21}(\beta-\hat\beta_{\mathcal{D}})+e_t\big).
\end{equation}
Thus
\begin{equation}
    f_{\mathcal{D}}(x)=z^\top\hat\beta_{\mathcal{D}}+\big(t^\top-z^\top A^{-1}B\big)\gamma_{\mathcal{D}}.
\end{equation}
Compared with the noise-free case, the additional term $H^{-1}e_t$ is the off-subspace noise amplification. It is not governed solely by $r$ without further assumptions on the nuisance design, covariance scale, or regularization.

\subsection{Proof of Corollary~\ref{cor:bayes_concept}}
\begin{proof}
The model $y=Z\beta+\varepsilon$, $\beta\sim\mathcal{N}(0,I_r)$, and $\varepsilon\sim\mathcal{N}(0,\sigma^2I_M)$ is conjugate Gaussian linear regression \cite{bishop2006pattern,murphy2012machine}. Completing the square in the posterior density gives
\begin{equation}
    \Sigma_{\mathcal{D}}=\left(I_r+\frac1{\sigma^2}Z^\top Z\right)^{-1},\qquad \mu_{\mathcal{D}}=\Sigma_{\mathcal{D}}\frac1{\sigma^2}Z^\top y.
\end{equation}
The posterior predictive mean at query $x$ is $\mu_{\mathcal{D}}^\top U^\top x$. Since
\begin{equation}
    \mu_{\mathcal{D}}=(Z^\top Z+\sigma^2 I_r)^{-1}Z^\top y,
\end{equation}
it coincides with the concept-space ridge estimator $(Z^\top Z+M\lambda I_r)^{-1}Z^\top y$ when $M\lambda=\sigma^2$, i.e., $\lambda=\sigma^2/M$.
\end{proof}

\subsection{Proof of Corollary~\ref{cor:ood_stability}}
\begin{proof}
The ambient ridge statement is exactly the invariance bound in Theorem~\ref{thm:concept_dimension}, after renaming constants:
\begin{equation}
    |f_{\mathcal{D}}(x+v)-f_{\mathcal{D}}(x)|\le K_{\mathrm{stab}}R_\perp(\delta)\|v\|\sqrt{\frac{r+\log(a_0/\delta)}{M}}.
\end{equation}
For the Bayes concept predictor, if $v\in\operatorname{span}(U)^\perp$, then $U^\top(x+v)=U^\top x$. Hence $\mu_{\mathcal{D}}^\top U^\top(x+v)=\mu_{\mathcal{D}}^\top U^\top x$, so the Bayes concept predictor is exactly invariant to nuisance perturbations.
\end{proof}

\subsection{Proof of Proposition~\ref{prop:identifiability}}
\begin{proof}
For each task $t$, the label model is $y=\langle U\beta_t,x\rangle+\varepsilon$. Since $\mathbb{E}[x\varepsilon\mid t]=0$,
\begin{equation}
    m_t=\mathbb{E}[xy\mid t]=\mathbb{E}[xx^\top]U\beta_t+\mathbb{E}[x\varepsilon\mid t]=\Lambda U\beta_t.
\end{equation}
Stacking the task-conditioned moments gives $G_{xy}=[m_1,\dots,m_T]=\Lambda U B$. Because $\Lambda\succ0$, $U$ has full column rank, and $B$ has rank $r$, we have $\operatorname{rank}(G_{xy})=r$. Moreover,
\begin{equation}
    \operatorname{span}(\Lambda^{-1}G_{xy})=\operatorname{span}(UB)=\operatorname{span}(U).
\end{equation}
If $\beta_t\stackrel{i.i.d.}{\sim}\mathcal{N}(0,I_r)$ and $T\ge r$, then $B$ has rank $r$ almost surely. The remaining ambiguity is only the choice of orthonormal basis within the recovered subspace: any two orthonormal bases of the same $r$-dimensional subspace differ by an orthogonal matrix $R$, giving $\tilde U=UR$ and $\tilde\beta_t=R^\top\beta_t$. Finally, under the zero-mean task prior, the unconditional first moment is
\begin{equation}
    \mathbb{E}[xy]=\mathbb{E}_\beta[\Lambda U\beta]=\Lambda U\mathbb{E}[\beta]=0,
\end{equation}
Therefore, task conditioning is necessary for this first-moment identification argument.
\end{proof}

\section{Additional Experimental Details and Analyses}
\subsection{Implementation Details}
\label{subsec:implementation}
\paragraph{Models and infrastructure.}
We run experiments on open-source language models, primarily Llama-3-8B and Qwen2.5-7B. To facilitate precise extraction and manipulation of internal representations, we implement all forward passes and causal interventions (activation patching) using the \texttt{transformer\_lens} library. All experiments are executed on one NVIDIA V100 GPU. Our software environment is built upon Python 3.10, PyTorch 2.0, and CUDA 11.8.

\paragraph{Dataset and prompt construction.}
Our task and relation data are derived from the CounterFact dataset. To ensure robust statistical estimation, we filter the dataset to retain only relations with a sufficient number of unique subject-target pairs (e.g., $N \ge 90$ depending on the specific experiment). When constructing $n$-shot in-context learning prompts, we randomly sample $n$ demonstration pairs (e.g., $n=12$) from the same relation pool, ensuring that the demonstration subjects are disjoint from the test query subject to prevent data leakage. For tasks requiring different prompt formats (e.g., symbolic, QA, or verbose), the same set of subject-target pairs is injected into distinct string templates to control for semantic equivalence.

\paragraph{Subspace estimation protocol.}
To estimate the concept subspace $\widehat{U}_{\ell^\star}$ at a target layer $\ell^\star$ (typically late-middle layers such as $26$ or $30$), we collect the residual stream activations at the final sequence position (the query token). Depending on the experiment, the subspace is extracted either via the cross-covariance between the hidden states $H$ and the target unembedding vectors $Y$ ($W_U$ corresponding to the first token of the correct answer), or via counterfactual ablation ($\Delta H = H_{\mathrm{task}} - H_{\mathrm{entity}}$). We then perform Singular Value Decomposition (SVD). Instead of fixing a hard-coded dimension, we dynamically determine the rank $r$ by selecting the top singular vectors that account for a predefined threshold of the explained variance (set to $98\%$ across all main experiments).

\begin{figure}[t]
  \centering
  \includegraphics[width=\linewidth]{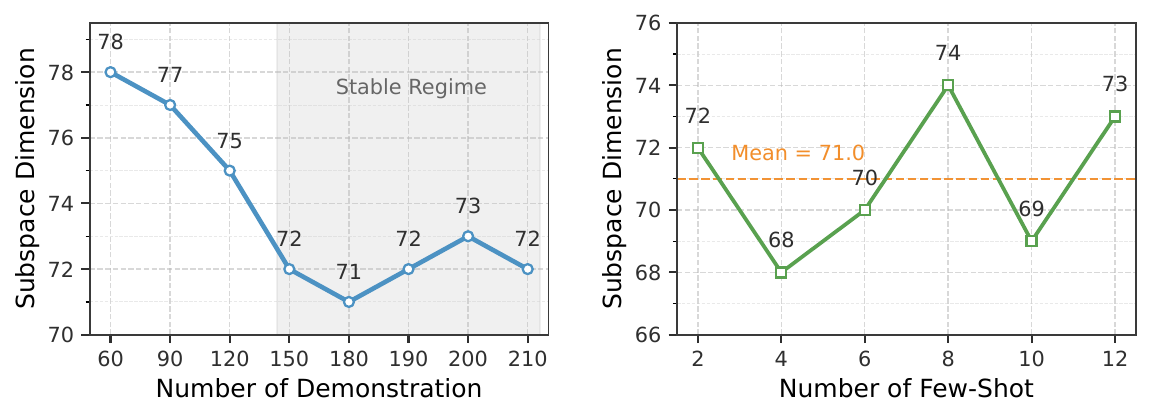}
  \caption{Dimension stability of concept-subspace estimates. The 98\% criterion selects a consistently low rank across the number of demonstrations and few-shot. }
  \label{fig:sensitivity}
\end{figure}

\subsection{Stability Analyses of Subspace Dimension.}
\label{subsec:stability_analyses}
To test whether the low rank reflects a stable task-aligned representation rather than an extraction artifact, we conduct the subspace estimation procedure at $\ell^\star=30$ under the same CounterFact setup while varying (i) the number of examples used for subspace estimation and (ii) the number of in-context demonstrations. In all cases, the estimated rank is the smallest $r$ satisfying ${\sum_{i=1}^{r}\sigma_i^2} / {\sum_i \sigma_i^2} \times 100 \% \ge 98\%$.
Figure~\ref{fig:sensitivity} (left) shows that the selected rank decreases from 69 to 61 as the estimation set grows and then stabilizes at 61--62 for dataset sizes 150--210. This is consistent with small-sample overestimation: under a conservative 98\% cutoff, limited data retain noisy tail directions, whereas more data suppress them without revealing a larger latent task manifold. Figure~\ref{fig:sensitivity} (right) shows a similarly stable pattern across prompt lengths: varying the few-shot count from 2 to 12 keeps the rank within 68--74 (mean 71.0), with no monotonic trend. Thus, more demonstrations strengthen evidence for the same relation but do not create new task directions.
 
Taken together, these robustness checks strengthen the main mechanistic claim. They do not imply that there is a single exact ``true'' rank recoverable from data, and the absolute number depends on the estimator and on the 98\% threshold. Rather, they establish a more important and defensible conclusion: across a reasonable range of estimation sizes and prompt configurations, the model consistently exhibits a sharply low-dimensional, task-aligned operating subspace, whose dimension stays in a narrow band from 60 to 70. This is the level of stability needed for the causal interpretation in the main experiments. The concept-only patching results are therefore not a brittle consequence of a single hand-picked extraction setting, but rather evidence for a stable geometric bottleneck in the model's task representation.

\begin{figure*}[t]
  \centering
  \includegraphics[width=0.49\textwidth]{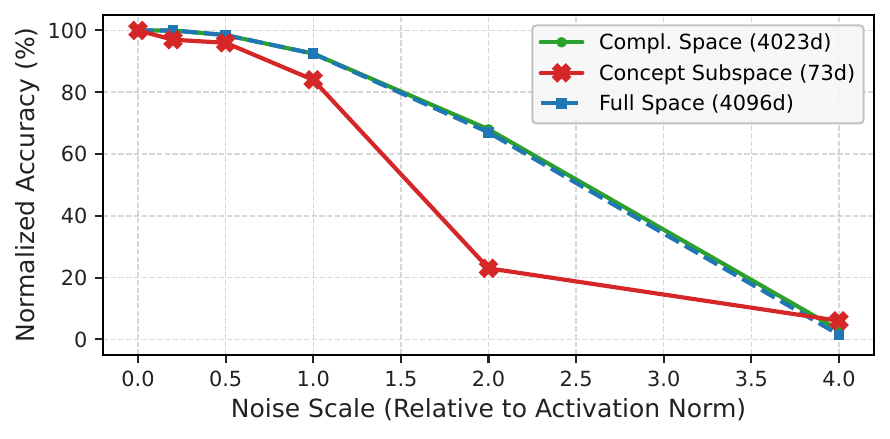}
  \hfill
  \includegraphics[width=0.49\textwidth]{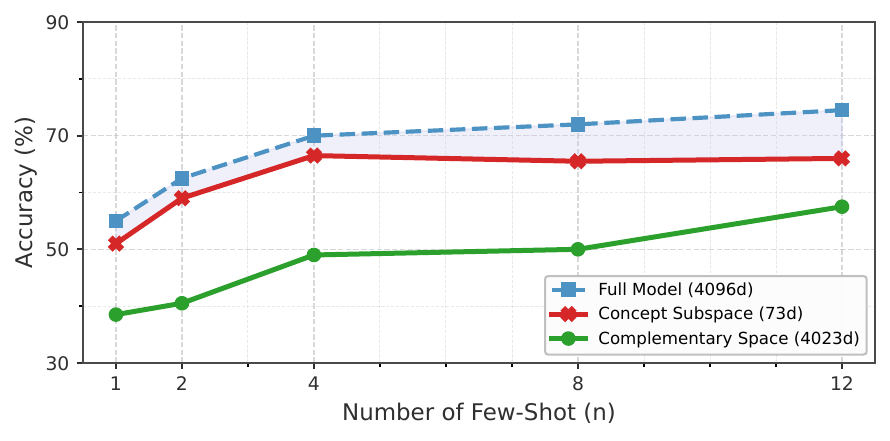}
  \caption{
  Subspace-selective sensitivity and few-shot scaling. \emph{Left:} concept-aligned noise degrades accuracy more sharply than matched-norm full-space or complementary-space noise. \emph{Right:} concept-subspace interventions recover most of the few-shot gains across shot counts.
  }
  \label{fig:robustness_to_noise_few_shot}
\end{figure*}

\subsection{Causal Sensitivity and Layerwise Emergence}
\label{subsec:causal_sensitivity}
\paragraph{Subspace-selective noise sensitivity.}
\label{subsec:Subspace_sensitivity}
We next test whether the model is selectively sensitive to perturbations in the learned concept subspace. In the same CounterFact multi-relation setting, we estimate a $73$-dimensional concept projector $P=\widehat{U}\widehat{U}^{\top}$ at the intervention layer using the $98\%$ explained-variance threshold. For correctly answered few-shot examples, we sample Gaussian noise $\xi$, decompose it into a concept-aligned component $P\xi$ and a complementary component $(I-P)\xi$, rescale each to the same norm ratio relative to the activation norm, and patch it into the residual stream; isotropic full-space noise with matched norm serves as a baseline. This isolates the effect of perturbation direction rather than magnitude.

Figure~\ref{fig:robustness_to_noise_few_shot} (left) shows a clear directional asymmetry. At small noise scales, all perturbations have a limited effect, but in the moderate-noise regime, concept-subspace noise degrades accuracy much more rapidly than complementary-space or full-space noise: around $2\times$ the activation norm, concept-aligned perturbations reduce accuracy to the low-20\% range, whereas the other two still retain roughly two-thirds accuracy. Only at the largest scale do all perturbations collapse performance. This pronounced sensitivity confirms that task-critical information is geometrically concentrated. Because the model can tolerate significant noise in the complementary space but rapidly degrades when the concept subspace is disrupted, we conclude that this low-dimensional subspace acts as the central causal bottleneck controlling the model's logical inference.

\paragraph{Few-shot scaling of subspace-mediated ICL gains.}
\label{appendix:scaling}
Using the same multi-relation CounterFact setup as above, we vary the number of few-shot examples $n\in\{1,2,4,8,12\}$ and ask whether the learned concept subspace remains the dominant carrier of recoverable few-shot gains as the prompt grows. For each query, we compute the activation change between the $n$-shot run and its zero-shot counterpart at the selected layer $\ell^\star$, decompose this change into the learned concept subspace $\widehat{U}_{\ell^\star}$ (73d) and its orthogonal complement, and inject each component into the zero-shot run. We then compare the resulting accuracy with the original full $n$-shot model.

Figure~\ref{fig:robustness_to_noise_few_shot} (right) shows a stable ordering across all shot counts: full-model performance is highest, concept-subspace recovery is consistently second, and complementary-space recovery is substantially weaker. Notably, the 73-dimensional concept subspace already recovers most of the few-shot benefit, increasing from 51\% at 1-shot to about 66\% from 4 shots onward, while the full model rises from 55\% to 75\%. This suggests that the core relation-level rule is quickly compressed into a low-dimensional state: once a few examples identify the task, most of the transferable gain is already available within the learned bottleneck.

The complementary space is not uninformative: its accuracy increases from 39\% to 58\% as $n$ grows, yet it remains substantially less predictive than the learned subspace. The result supports a \emph{concentrated-mediator} claim rather than an exclusivity claim: the concept subspace carries the dominant recoverable task signal, while off-subspace directions likely encode auxiliary information such as calibration, entity-specific disambiguation, or other distributed contextual effects that become more useful at larger $n$. The widening gap between the full model and concept-only intervention at higher shot counts is therefore informative rather than contradictory, indicating that additional demonstrations refine the computation beyond the current linear subspace estimate.

\paragraph{Layerwise emergence of a causal concept bottleneck.}
\label{appendix:layerwise_analysis}
Using the same CounterFact multi-relation setting, we estimate a layer-specific subspace $\widehat{U}_{\ell}$ at each layer with the same $98\%$ explained-variance criterion, and repeat the relation-swap intervention with three replacements: the full residual state, only the projection onto $\widehat{U}_{\ell}$, or only its orthogonal complement. We report the resulting layerwise override success rate.

Figure~\ref{fig:causal_and_geometry} (left) shows a clear stagewise pattern. In shallow layers, all three interventions remain near the floor, indicating that relational identity is not yet represented in a causally usable form. In the middle of the network, full-state replacement rises first while concept-subspace replacement lags behind, suggesting that task-relevant information is present before it is compressed into a compact low-dimensional representation. In later layers, concept-subspace replacement becomes consistently strong and approaches full-state performance, whereas complementary-space replacement remains near zero. This separation supports a late-layer causal bottleneck: most manipulable task identity is concentrated in a small set of task-aligned directions rather than distributed across the ambient 4096-dimensional residual stream. Since full-state replacement remains somewhat stronger, we interpret the learned subspace as a dominant but not exhaustive mediator. 
\begin{figure*}[t]
  \centering
  \includegraphics[width=0.49\textwidth]{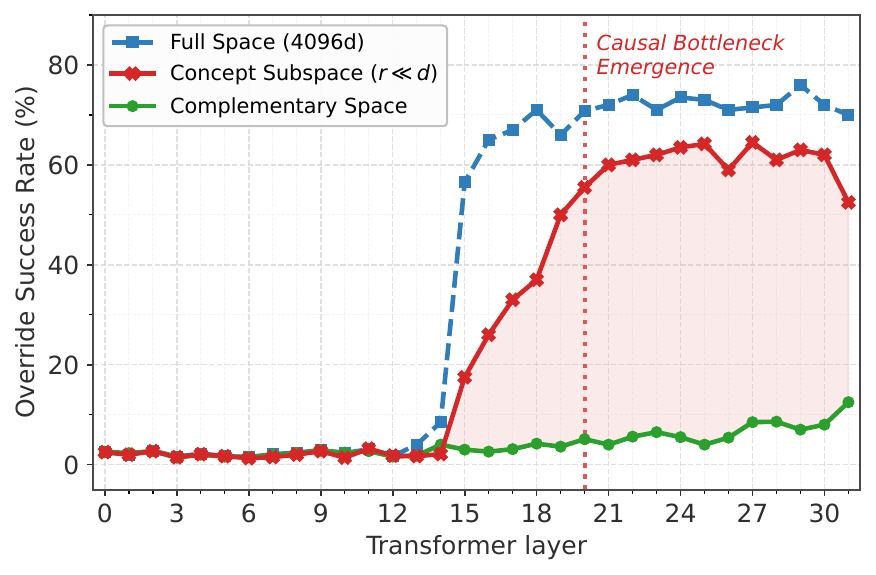}
  \hfill
  \includegraphics[width=0.49\textwidth]{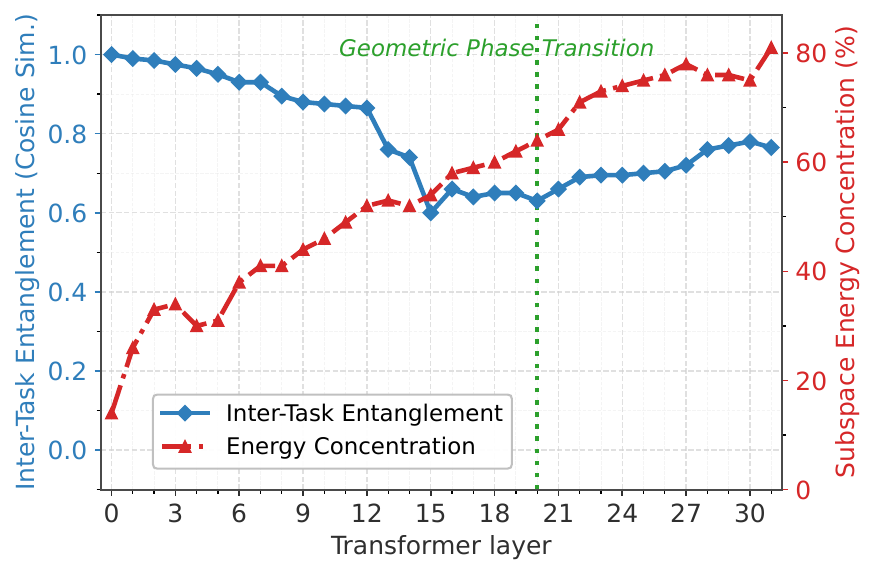}
  \caption{
    Layerwise emergence of a causal concept bottleneck. \emph{Left:} override success for full-state, concept-subspace, and complementary-space interventions across layers. \emph{Right:} inter-task entanglement decreases while concept-subspace energy concentration increases with depth.
  }
  \label{fig:causal_and_geometry}
\end{figure*}

We also track two geometric diagnostics. Let $P_{\ell}=\widehat{U}_{\ell}\widehat{U}_{\ell}^{\top}$ and let $h_{\ell}^{(R,s)}$ be the query-token residual state for relation $R$ and subject $s$. For subject-matched relation pairs, define inter-task entanglement $E_{\ell} = \mathbb{E}_{s,R_1\neq R_2} \![\cos\!( P_{\ell}h_{\ell}^{(R_1,s)}, P_{\ell}h_{\ell}^{(R_2,s)} ) ]$ and energy concentration $C_{\ell} = 100\cdot \mathbb{E}_{s,R} \![ {\|P_{\ell}h_{\ell}^{(R,s)}\|_2} / {\|h_{\ell}^{(R,s)}\|_2} ]$.
Here $E_{\ell}$ measures inter-task entanglement in the learned subspace, while $C_{\ell}$ measures the fraction of residual-state norm captured by that subspace.
Figure~\ref{fig:causal_and_geometry} (right) shows the corresponding geometric picture. In shallow layers, subject-matched prompts for different relations have high inter-task entanglement, consistent with representations that are still dominated by shared entity or surface-context information. As depth increases, this entanglement drops from its shallow-layer value, while the concept-subspace concentration increases. Thus, the same layer range in which concept-subspace interventions become effective is also where relation-specific directions become more separated and residual variation is increasingly captured by the learned subspace. We interpret this as geometric consolidation of task identity rather than a strict phase transition: the trend supports the emergence of a compact causal bottleneck, but does not imply that all task-relevant information is confined to $\widehat{U}_{\ell}$.

\subsection{Task Invariance and Format Disentanglement}
\label{subsec:task_invariance}
\paragraph{Task-format disentanglement in concept coordinates.}
To test whether the learned subspace captures task identity rather than superficial prompt format, we construct a factorial dataset with five relation tasks and three prompt formats: symbolic, verbose natural language, and question answering. For each relation $r$, format $f$, and query instance $i$, we extract an incremental task vector at the selected layer $v_{\ell^\star}^{(r,f,i)} = h_{\ell^\star,\mathrm{few}}^{(r,f,i)}
- h_{\ell^\star,\mathrm{zero}}^{(f,i)}$, and decompose it using the learned orthonormal concept basis $\widehat U_{\ell^\star}$: $\beta^{(r,f,i)} = \widehat U_{\ell^\star}^{\top} v_{\ell^\star}^{(r,f,i)}$ and $v_{\perp}^{(r,f,i)} = (I-\widehat U_{\ell^\star}\widehat U_{\ell^\star}^{\top}) v_{\ell^\star}^{(r,f,i)}$.
We compare the full residual vector $v$, the concept coordinates $\beta$, and the complementary component $v_\perp$. For each representation, we compute silhouette scores using either relation labels or prompt-format labels. A format-invariant task representation should increase task separability while suppressing format separability. \looseness=-1

Table~\ref{tab:task_format_silhouette} and Figure~\ref{fig:format_invariant_geometry} reveal a striking geometric disentanglement. Compared to the ambiguous full space, the low-dimensional concept subspace extracts a format-invariant task manifold: task representations form dense, well-separated clusters ($0.411$), while format noise is erased ($-0.007$), forcing identical tasks from disparate templates to perfectly overlap. Conversely, in the complementary space, task identity collapses ($0.022$), leaving the logical rule drowned out by dominant surface-format residues ($0.173$). This diagnostic supports interpreting $\widehat U_{\ell^\star}$ as a task-aligned concept subspace rather than a generic low-rank projection. It also explains why complementary-subspace patching is ineffective in our intervention experiments: the complement preserves format-related information but contains little of the relation-level signal needed to restore in-context prediction. 
\begin{table}[t]
  \caption{Task–format separability across representation spaces. Concept coordinates improve relation separability while suppressing format separability.}
    \centering
    \small
    \begin{tabular}{lcc}
        \hline
        Space & Task silhouette & Format silhouette \\
        \hline
        Full space  & 0.112 & 0.129 \\
        Concept subspace  & 0.411 & -0.007 \\
        Complementary space  & 0.022 & 0.173 \\
        \hline
    \end{tabular}
    \label{tab:task_format_silhouette}
\end{table}

\begin{figure*}[t]
  \centering
  \includegraphics[width=\textwidth]{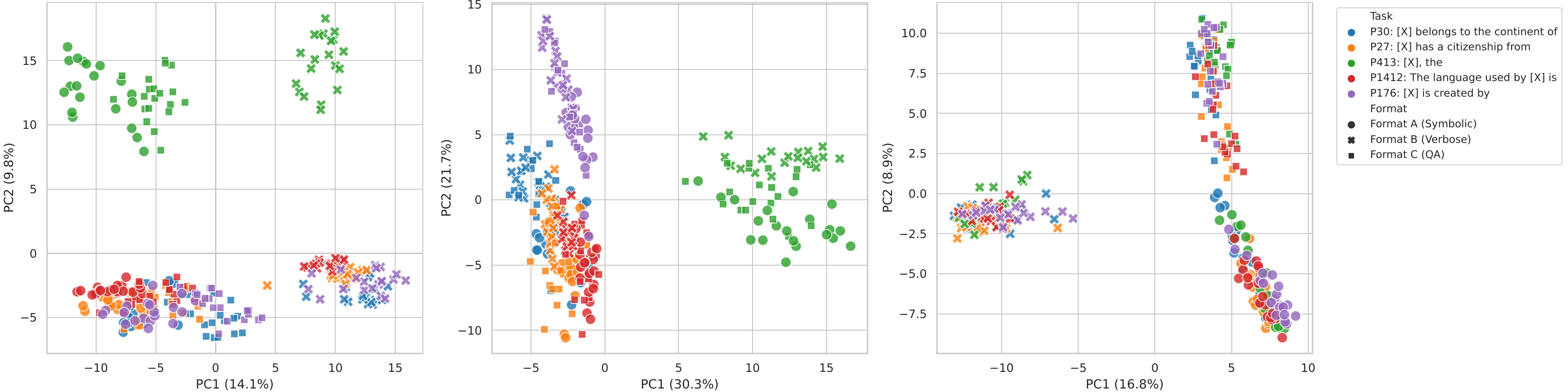}
  \caption{
    Task–format geometry after concept-subspace projection. PCA of task vectors shows that concept coordinates cluster primarily by relation, whereas complementary directions retain more prompt-format structure. \emph{Left:} full space; \emph{Middle:} concept subspace; \emph{Right:} complementary space. \looseness=-1
  }
  \label{fig:format_invariant_geometry}
\end{figure*}

\paragraph{Cross-format alignment of task representations.}
We test whether the learned subspace captures task semantics rather than the prompt template. Fixing the same underlying relation and few-shot exemplars, we rewrite each prompt in three formats (symbolic, QA, and natural language), extract the query-token residual at layer $\ell^\star$, and compare cross-format cosine similarity in the full residual stream, the learned concept subspace $\widehat{U}_{\ell^\star}$ (73d), and its orthogonal complement.

Figure~\ref{fig:subspace_similarity} (left) shows a consistent ordering: similarity is highest in the concept subspace (0.98), lower in the full state (0.89), and lowest in the complementary space (0.79). This suggests that the full residual mixes shared task information with format-specific variation, while much of the surface-form noise lies outside $\widehat{U}_{\ell^\star}$. The near-perfect alignment inside the learned subspace indicates that it preserves the same relation-level computation across prompt realizations.

This also validates the subspace estimator: if $\widehat{U}_{\ell^\star}$ primarily captured lexical or stylistic regularities, projecting onto it would not improve cross-format agreement over the full state. Instead, projection removes format-dependent variation and aligns representations. We therefore interpret $\widehat{U}_{\ell^\star}$ as a format-stable carrier of task information under controlled prompt rewrites while avoiding a stronger claim of robustness under arbitrary distribution shifts.

\begin{figure}[t]
  \centering
  \includegraphics[width=\linewidth]{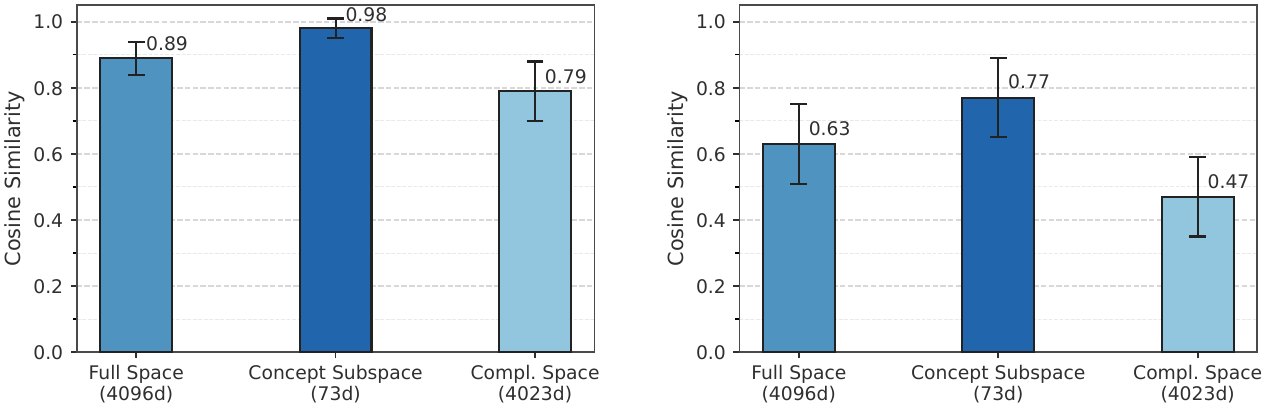}
  \caption{
    Format invariance and corruption stability of concept representations. \emph{Left:} cross-format cosine similarity. \emph{Right:} clean–corrupted representation similarity under label corruption.
  }
  \label{fig:subspace_similarity}
\end{figure}

\paragraph{Stability under label corruption.}
\label{appendix:demo_stability}
We test whether the learned concept subspace remains stable when demonstrations are severely corrupted. Starting from a clean few-shot prompt, we construct a noisy version by replacing demonstration labels with random unrelated tokens while keeping the query and prompt structure fixed. At the selected layer, we extract the query-token residual state from both runs and compare clean--noisy cosine similarity in three spaces: the full residual stream, the learned 73-dimensional concept subspace, and its orthogonal complement.

Figure~\ref{fig:subspace_similarity} (right) shows a consistent ordering: similarity is highest in the concept subspace (0.77), lower in the full state (0.63), and lowest in the complementary space (0.47). This suggests that the full residual mixes task information with corruption-induced noise, while much of that noise is diverted into off-subspace directions. The higher similarity inside the learned subspace indicates that relation-level information is more stable than the ambient representation under label corruption.

This supports interpreting the learned subspace as a denoised carrier of task information rather than a prompt-specific artifact: projecting onto it increases agreement between clean and corrupted representations instead of reducing it. We view this as representation-level robustness under controlled corruption, not a general guaranty under arbitrary adversarial shifts. Still, alongside the patching and swapping results, this further strengthens the claim that the learned concept subspace captures task-level structure rather than prompt-local statistics.

\subsection{Prompt Contamination and Few-Shot Inference Dynamics}
\label{subsec:prompt_ontamination}
\paragraph{Prompt Contamination and Concept-Vector Alignment.}
We next test whether failure under prompt contamination can be explained geometrically inside the learned concept subspace. Let $P=\widehat{U}\widehat{U}^{\top}$ denote the projector onto the learned concept subspace, and let $E_{\mathrm{target}}$ be the unembedding matrix of the correct answer token. We define an answer-anchored reference vector $\beta_{\mathrm{ideal}}=\frac{P E_{\mathrm{target}}}{\|P E_{\mathrm{target}}\|_2}$ and fix the total prompt budget at 10 shots while progressively replacing relevant demonstrations with cross-domain irrelevant examples. For each contamination level, we run autoregressive inference, extract the induced task direction $\beta_{\mathrm{ctx}}$ in the concept subspace, and measure its absolute alignment to the reference $a = \bigl|\cos(\beta_{\mathrm{ctx}}, \beta_{\mathrm{ideal}})\bigr|$, together with the final prediction accuracy.  

Figure~\ref{fig:irrelevant_shots} shows that internal geometric drift and external performance degradation closely co-vary. As the number of irrelevant shots increases, alignment to the target answer direction decreases steadily, and accuracy declines in parallel, with a Pearson correlation of $r=0.787$. The degradation is initially gradual but becomes much sharper in the high-contamination regime, where both curves bend downward. This pattern suggests that irrelevant demonstrations do not merely add generic noise; they systematically rotate the inferred task vector away from the answer-consistent region of the learned concept subspace. The result strengthens the mechanistic interpretation of $\widehat{U}$: the same subspace that supports recovery and concept swapping also predicts robustness limits under prompt contamination. At the same time, the alignment--accuracy relation is strong but not perfect, so this quantity is best viewed as a compact diagnostic of task fidelity rather than a complete account of all failure modes. 
\begin{figure}[t]
    \centering
    \includegraphics[width=0.8\linewidth]{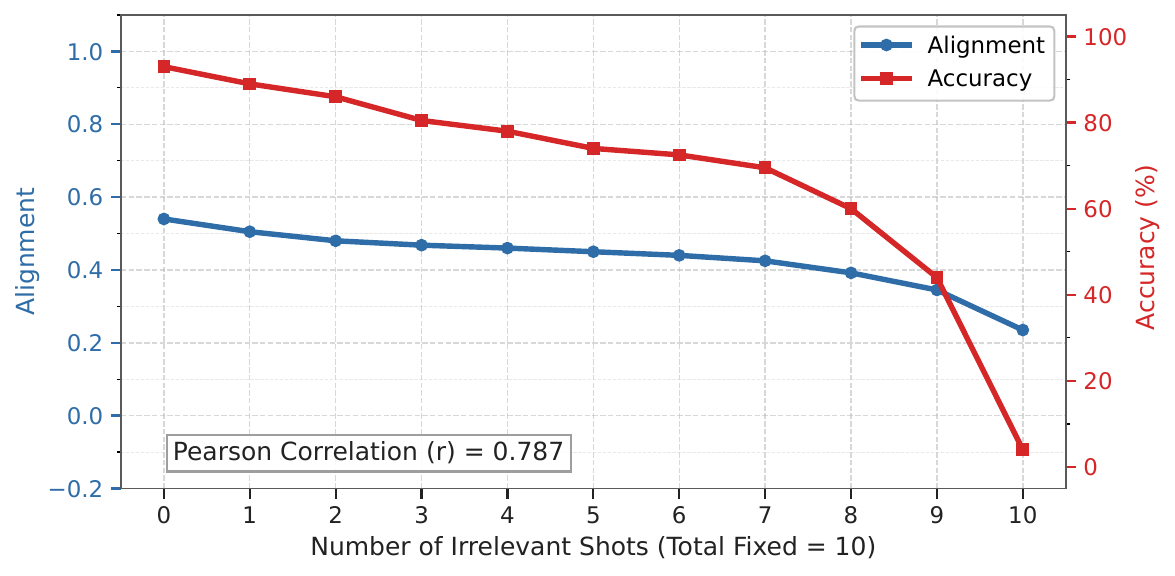}
    \caption{Prompt contamination rotates the inferred concept direction. Alignment in concept space and prediction accuracy jointly decline as irrelevant demonstrations replace relevant shots.}
    \label{fig:irrelevant_shots}
\end{figure}

\subsection{Experimental Results with Qwen2.5-7B}
\label{sec:qwen_results}
We further test whether the observed concept-subspace effect is specific to Llama-3-8B~\cite{grattafiori2024llama} by repeating the corresponding experiments on Qwen2.5-7B \cite{qwen_qwen25_2024}. Unless otherwise stated, we use the same CounterFact-derived multi-relation prompt construction, clean--corrupted intervention protocol, subspace extraction criterion, and evaluation metrics as in the main experiments. For Qwen2.5-7B, the intervention layer is the 27th transformer layer, and all subspace estimates and patching interventions in this subsection are performed at that layer.

Figure~\ref{fig:qwen_layerwise_and_scree} first shows the layerwise localization and low-rank extraction results. The clean-to-corrupted full-state patching curve again exhibits a late-layer pattern: early layers provide little recovery, while recovery becomes substantial only in the later part of the network and approaches the clean reference level near the final layers. This indicates that, as in Llama-3-8B, task-level information is not uniformly available throughout the stack but becomes causally accessible after sufficient contextual processing. At the selected 27th layer, the explained-cross-variance spectrum is sharply concentrated. The same 98\% threshold selects a 45-dimensional subspace out of the 3584-dimensional residual stream, suggesting that the task-aligned signal remains highly compressed in this model as well. \looseness=-1

Figure~\ref{fig:qwen_main_interventions} tests whether this low-dimensional subspace is causally relevant. In the patching experiment, the baseline accuracy is 32.5\%, while the few-shot upper bound is 57.0\%. Full-space patching reaches 57.0\%, corresponding to 100.0\% recovery of the clean--baseline gap. Patching only the 45-dimensional concept subspace reaches 53.0\%, corresponding to 83.7\% recovery, despite modifying only about 1.1\% of the residual dimensions. In contrast, patching the 3539-dimensional complementary space reduces accuracy to 21.0\%, yielding a negative recovery rate. We therefore do not interpret the complementary intervention as carrying useful task information in this setting. Rather, it provides a strong control showing that the recovery is not explained by the patched dimensions. 

The concept-swap results show the same qualitative specificity. Full-space replacement achieves 69.0\% override success for $R_1 \rightarrow R_2$ and 61.0\% for $R_2 \rightarrow R_1$. Replacing only the learned concept subspace remains close to this level, with 59.0\% and 56.0\% override success, respectively. By contrast, complementary-space replacement is largely ineffective, reaching only 8.0\% and 5.0\%. Thus, on Qwen2.5-7B, relation identity is again concentrated in the learned task-aligned subspace rather than in the high-dimensional orthogonal complement.

Overall, these results support the same restrained interpretation as the main experiments. The exact selected rank and intervention layer differ across models, and the experiment does not imply a universal concept dimension or a complete circuit-level explanation. Nevertheless, Qwen2.5-7B exhibits the same qualitative pattern: a small task-aligned residual subspace is sufficient to recover most of the clean--corrupted gap and to redirect predictions under concept swaps, whereas the complementary space has little corresponding effect.
\begin{figure*}[!t]
  \centering
  \includegraphics[width=0.49\textwidth]{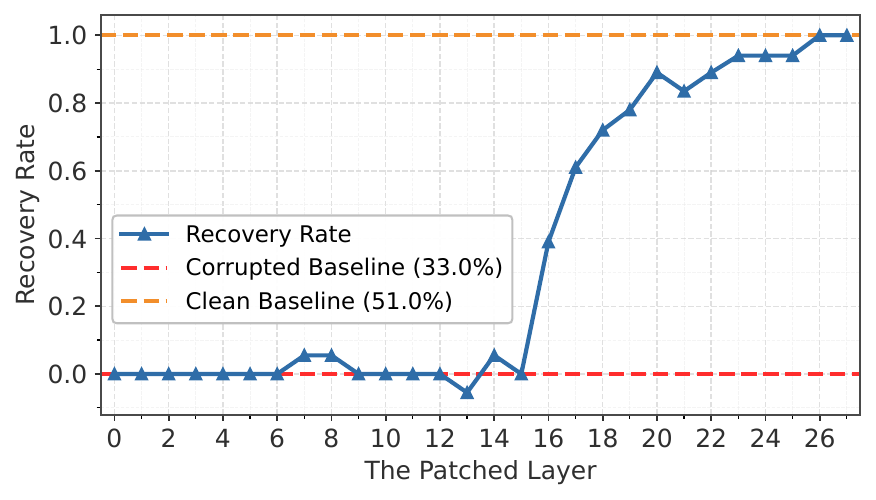}
  \hfill
  \includegraphics[width=0.49\textwidth]{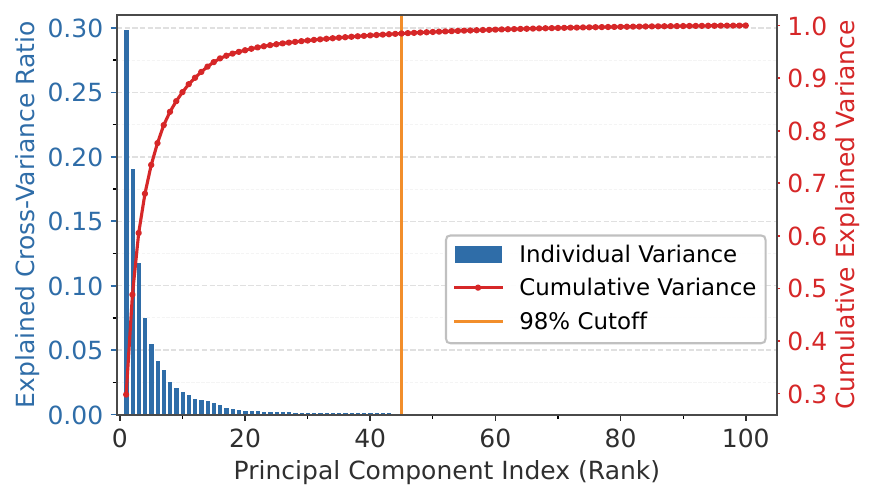}
  \caption{
    Layerwise localization and low-rank extraction of task-aligned residual directions. \emph{Left}: clean-to-corrupted patching localizes layers where task information is causally accessible. \emph{Right}: the 98\% explained-cross-variance threshold selects a low-dimensional concept subspace at layer 27.
  }
  \label{fig:qwen_layerwise_and_scree}
\end{figure*}
\begin{figure*}[t]
    \centering
    \includegraphics[width=0.49\textwidth]{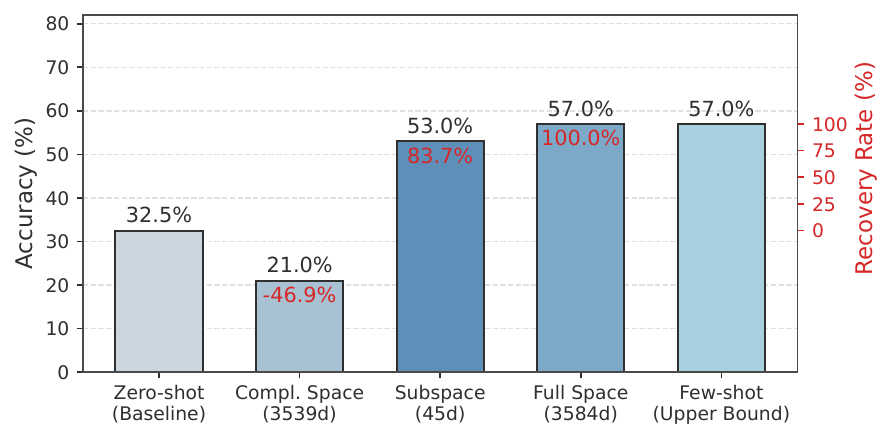}
    \hfill
    \includegraphics[width=0.49\textwidth]{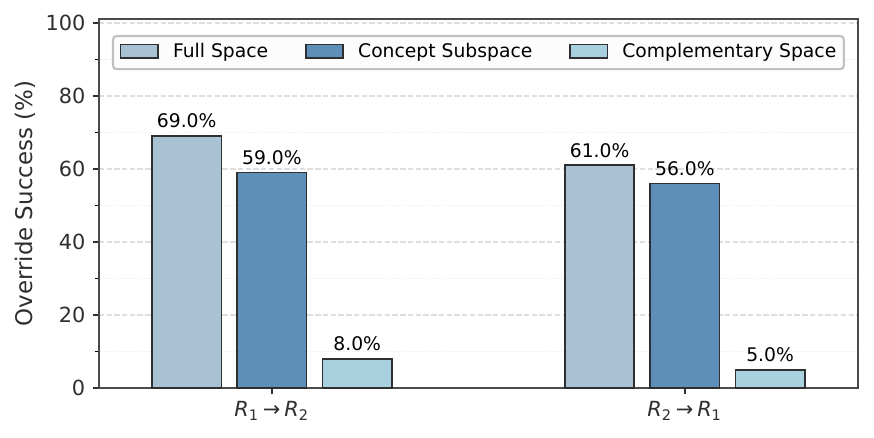}
    \caption{
    Concept-subspace patching and swapping. \emph{Left:} the learned subspace recovers most of the clean-corrupted gap. \emph{Right:} subspace swaps redirect predictions toward the injected relation.
    }
    \label{fig:qwen_main_interventions}
\end{figure*}

\subsection{Experimental Results on Cross-Lingual Task}
\label{sec:crosslingual}
We use a controlled cross-lingual task to test whether concept-subspace mediation extends beyond factual-relation ICL to a more abstract rule: selecting the output language. From MUSE English--French and English--Spanish dictionaries~\cite{lample2018word}, we construct word triples and retain only examples whose French and Spanish targets are single Llama-3 tokens, lexically non-overlapping, and not very short or ambiguous. Prompts contain no explicit language tags such as \texttt{Fr:} or \texttt{Es:}; all demonstrations use the same format, e.g., \texttt{water -> eau}, so the model must infer the target-language rule from context. Let $R_1$ denote French and $R_2$ Spanish. The ``corrupted'' run below is a counterfactual competing-rule run, not a malformed prompt. All subspace estimation and interventions are performed on Llama-3-8B at layer 30, at the final query token.

Figure~\ref{fig:llama_layerwise_and_scree_crosslingual} localizes the cross-lingual rule representation. Full-state patching has little effect in early layers and becomes effective only in later layers, matching the late emergence observed in the CounterFact experiments. At layer 30, the cross-variance spectrum is highly concentrated: the same 98\% explained-cross-variance criterion selects a 92-dimensional subspace out of the 4096-dimensional residual stream. Thus, the target-language rule appears in a compact set of activation directions, although this does not imply that the full translation computation is low-dimensional.

Figure~\ref{fig:llama_main_interventions_crosslingual} tests the causal role of this subspace. When evaluated against the French target, the competing-rule baseline is essentially at the floor, while native few-shot French and Spanish contexts reach 73.0\% and 72.0\% accuracy, respectively. Full-state patching raises accuracy to 52.6\%, recovering 72.0\% of the clean--competing gap. Patching only the 92-dimensional concept subspace reaches 29.6\% accuracy, or 40.5\% recovery, despite modifying only $92/4096$ residual dimensions. In contrast, patching the 4004-dimensional orthogonal complement gives only 0.7\% accuracy and 1.0\% recovery. The effect is therefore tied to task-aligned directions rather than intervention size.

The swap experiment gives a sharper test of directional control. Here $R_1 \!\to\! R_2$ denotes overriding a French-context run toward Spanish output, and $R_2 \!\to\! R_1$ the reverse. Full-state replacement achieves 72.0\% override success in both directions. Replacing only the learned cross-lingual subspace remains close to this level, with 66.7\% for $R_1 \!\to\! R_2$ and 65.3\% for $R_2 \!\to\! R_1$, whereas complementary-space replacement yields 0.0\% in both directions. Thus, the learned subspace carries most of the target-language identity needed to redirect generation. The remaining gap in patching accuracy suggests that lexical, calibration, or decoding-related factors outside this single-layer linear subspace also contribute. Overall, the cross-lingual results extend the main finding beyond CounterFact-style relations: a low-dimensional, task-aligned residual subspace acts as a major causal mediator of language-rule selection, but not as a complete circuit-level explanation of cross-lingual generation.

\begin{figure*}[t]
  \centering
  \includegraphics[width=0.49\textwidth]{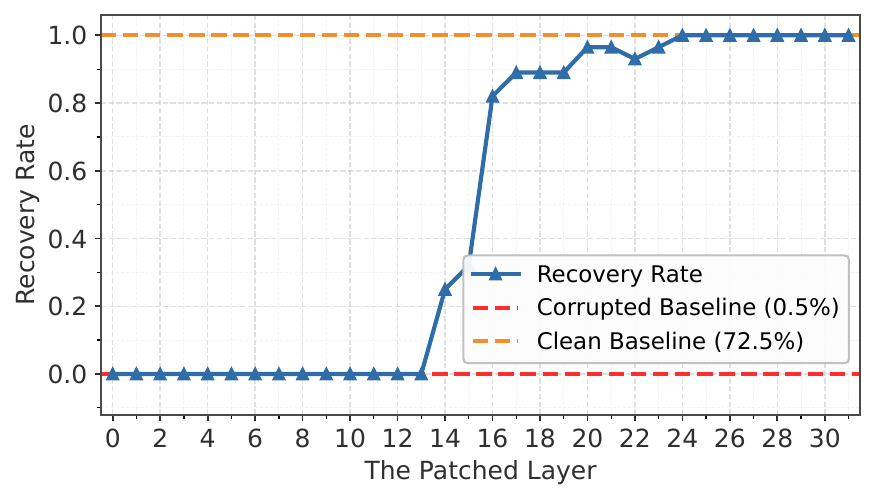}
  \hfill
  \includegraphics[width=0.49\textwidth]{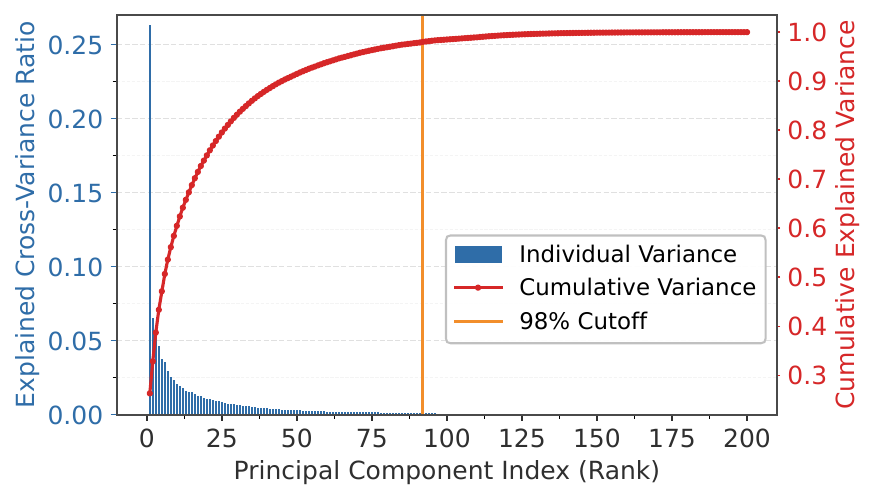}
  \caption{
    Layerwise localization and low-rank extraction of task-aligned residual directions. \emph{Left}: clean-to-corrupted patching localizes layers where task information is causally accessible. \emph{Right}: the 98\% explained-cross-variance threshold selects a low-dimensional concept subspace at layer 30.
  }
  \label{fig:llama_layerwise_and_scree_crosslingual}
\end{figure*}

\begin{figure*}[t]
    \centering
    \includegraphics[width=0.49\textwidth]{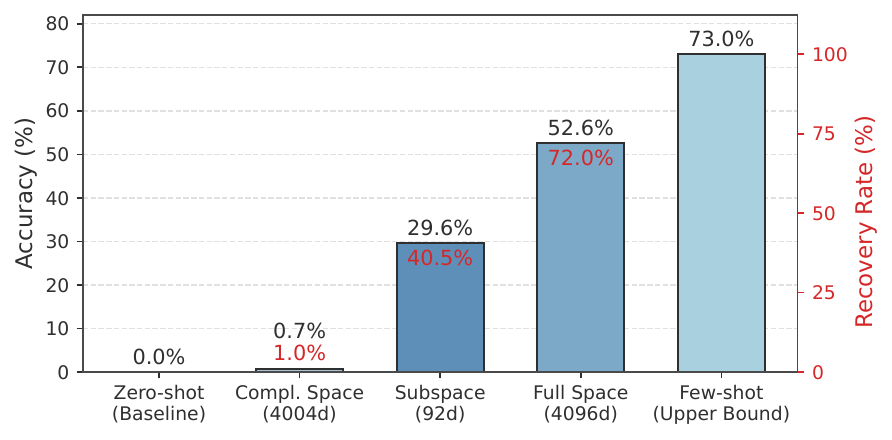}
    \hfill
    \includegraphics[width=0.49\textwidth]{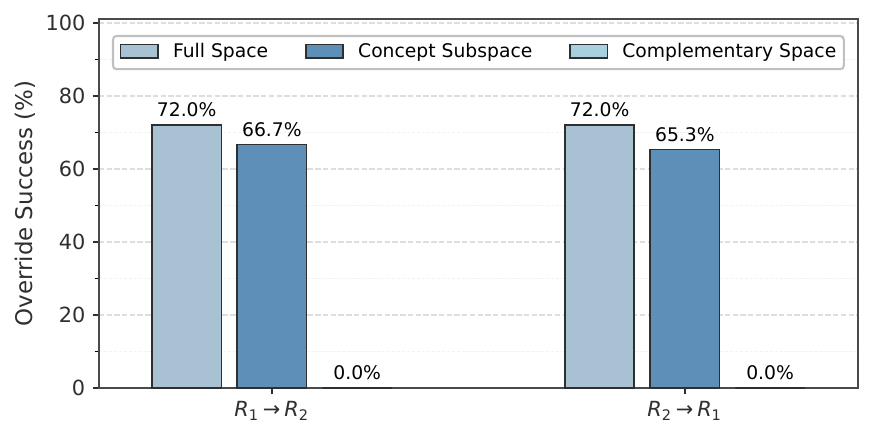}
    \caption{
    Concept-subspace patching and swapping. \emph{Left:} the learned subspace recovers most of the clean-corrupted gap. \emph{Right:} subspace swaps redirect predictions toward the injected relation.
    }
    \label{fig:llama_main_interventions_crosslingual}
\end{figure*}

\section{Broader Impacts}
This work is a foundational study of in-context learning mechanisms rather than a deployment-oriented system. Its main potential benefit is to make model behavior more inspectable. By identifying low-dimensional task-aligned activation subspaces, the proposed framework can help diagnose whether a model is using stable task structure or relying on superficial prompt artifacts. Such diagnostics may support more targeted robustness evaluation, prompt auditing, and controlled intervention methods, especially in settings where few-shot behavior is difficult to interpret from outputs alone.

The same mechanism also has possible misuse risks. If task-level activation subspaces can be identified and manipulated, they could be used to steer model behavior in ways that are less visible than ordinary prompt changes. Moreover, because our interventions recover a dominant but incomplete mediator, they should not be treated as guarantees of correctness, fairness, or safety. A model may appear to follow an intended concept while still relying on biased, entity-specific, or format-dependent features outside the recovered subspace.

The empirical setting also matters for impact. Our relation tasks and pretrained model activations may reflect biases present in factual datasets and pretraining corpora. Compact concept directions could make such biases easier to audit, but they could also concentrate and transfer undesirable associations if used uncritically. Finally, although our experiments are far less costly than training foundation models, collecting activations and running intervention sweeps still incur compute costs. We therefore view the method as most valuable when it replaces broad trial-and-error testing with more precise mechanistic evaluation.

\end{document}